\documentclass{article}

 \usepackage[main,final]{neurips_2026}
 \usepackage{enumitem}

\usepackage[utf8]{inputenc} 
\usepackage[T1]{fontenc}    
\usepackage{url}            
\usepackage{wrapfig}
\usepackage{booktabs}       
\usepackage{amsfonts}       
\usepackage{amsmath}
\usepackage{amsthm} 
\providecommand{\ExpTableFont}{}
\providecommand{\mathbbm}[1]{\mathbf{#1}} 
\usepackage{algorithm}    
\usepackage{algpseudocode}
\usepackage{graphicx} 
\usepackage{multirow}
\usepackage{nicefrac}       
\usepackage{microtype}      
\usepackage{xcolor}         
\newtheorem{theorem}{Theorem}
\newtheorem{assumption}{Assumption}
\newtheorem{lemma}{Lemma}
\newtheorem{corollary}{Corollary}

\title{FAST-Brain: A Flow-Aligned Spatio-Temporal Surrogate Brain Model}

\usepackage[colorlinks=true, linkcolor=blue, citecolor=blue, urlcolor=blue]{hyperref}

\author{%
  Shucheng Liu \\
  University of North Carolina\\
  at Chapel Hill\\
  \texttt{shucheng@unc.edu}
  \And
  Chengchun Shi \\
  London School of Economics\\
  and Political Science\\
  \texttt{c.shi7@lse.ac.uk}
  \AND
  Kai Zhang \\
  University of North Carolina\\
  at Chapel Hill\\
  \texttt{zhangk@email.unc.edu}
  \And
  Hongtu Zhu \\
  University of North Carolina\\
  at Chapel Hill\\
  \texttt{htzhu@email.unc.edu}
}

\hypersetup{pdftitle={FAST-Brain: A Flow-Aligned Spatio-Temporal Surrogate Brain Model},pdfauthor={Shucheng Liu, Chengchun Shi, Kai Zhang, Hongtu Zhu}}

\begin{document}

\maketitle

\begin{abstract}
Modeling resting-state functional magnetic resonance imaging (rs-fMRI) data is crucial for understanding brain-wide neural activity. However, traditional methods struggle to capture complex temporal dynamics over long horizons, to account for the brain's anatomical spatial structure, and to model high-dimensional ambient signals that lie on a low-dimensional intrinsic subspace. We propose FAST-Brain, a unified flow-aligned spatio-temporal surrogate brain model that addresses all three challenges. At its core is a flow-aligned generative framework that directly predicts the clean blood-oxygen-level-dependent (BOLD) signal, paired with a graph convolutional network that captures spatial structural constraints and a Transformer that models long-range temporal dependencies. Theoretically, we show that under a low-dimensional subspace assumption,  the approximation error of our model scales with the intrinsic dimension rather than the ambient dimension, which justifies our direct modeling of the BOLD signal. Extensive experiments on synthetic and Human Connectome Project datasets demonstrate that FAST-Brain achieves state-of-the-art performance in functional connectivity recovery, surrogate-derived perturbation influence (SPI) estimation, and recovery of the implicit low-dimensional signal subspace.
\end{abstract}

\section{Introduction}\label{sec:intro}

The brain is a complex network of interconnected regions of interest (ROIs) that jointly support
cognition and behavior \citep{park2013structural, deco2021revisiting, zhu2023statistical, feng2026towards}.
Resting-state functional magnetic resonance imaging (rs-fMRI) provides a non-invasive window
into brain-wide blood-oxygen-level-dependent (BOLD) dynamics and inter-regional dependencies
\citep{Biswal1995, van2010exploring, calhoun2014chronnectome,lindquist2025statistics,kang2026statistical}. Recent data-driven methods aim
to learn the generative process of BOLD signals to construct high-fidelity surrogate brain
models, or digital twins \citep{Thomas2022, Lu2023, Dong2024, Luo2025}. Such models can
reproduce realistic neural dynamics and functional dependencies, and infer effective interactions
among brain regions through non-invasive computation \citep{Friston2011, Luo2025}.

However, existing models suffer from three limitations. First, many models struggle to capture
\textbf{long-range temporal dependencies}. Existing surrogate models rely on recurrent neural
networks (RNNs) or multilayer perceptrons \citep{tu2019state, perich2020inferring, Luo2025},
which fail to model the high-dimensional, long-horizon temporal structure of BOLD signals.

Second, existing models insufficiently account for the \textbf{spatial structure} of
brain dynamics. Cross-regional interactions are shaped by anatomical topology, including
white-matter structural connectivity, fiber lengths, and macroscopic functional network
organization \citep{hagmann2008mapping, honey2009predicting}. While several studies
incorporate structural connectivity or white-matter information into fMRI analysis
\citep{zhu2014fusing, wein2022forecasting, han2024brainode}, these
approaches rely on a single spatial constraint. Few existing models jointly
integrate multiple complementary spatial priors like structural connectivity, fiber lengths,
and functional network modules into the signal generation process. Moreover,
brain dynamics are not simple deterministic transitions but complex stochastic systems
driven by latent cognitive states and intrinsic noise \citep{deco2013resting,
breakspear2017dynamic}, causing standard temporal models to smooth out variability and
accumulate errors over long horizons. 

Third, existing diffusion- and flow-based fMRI generators do not fully exploit the \textbf{intrinsic
low-dimensional structure} of brain signals. These models parameterize high-dimensional denoising
residuals or velocity fields \citep{tew2025t2idiff, tew2026functional}, yet rs-fMRI signals---though
observed in a high-dimensional ROI-time space---are believed to concentrate near a low-dimensional
subspace shaped by anatomical and dynamical constraints \citep{vincent2010stacked,lindquist2025statistics, gallego2017neural,
pezon2024linking}. Learning targets in the ambient space can obscure this structure and reduce
sample efficiency.

\begin{figure}[t]
    \centering
    \includegraphics[width=\textwidth]{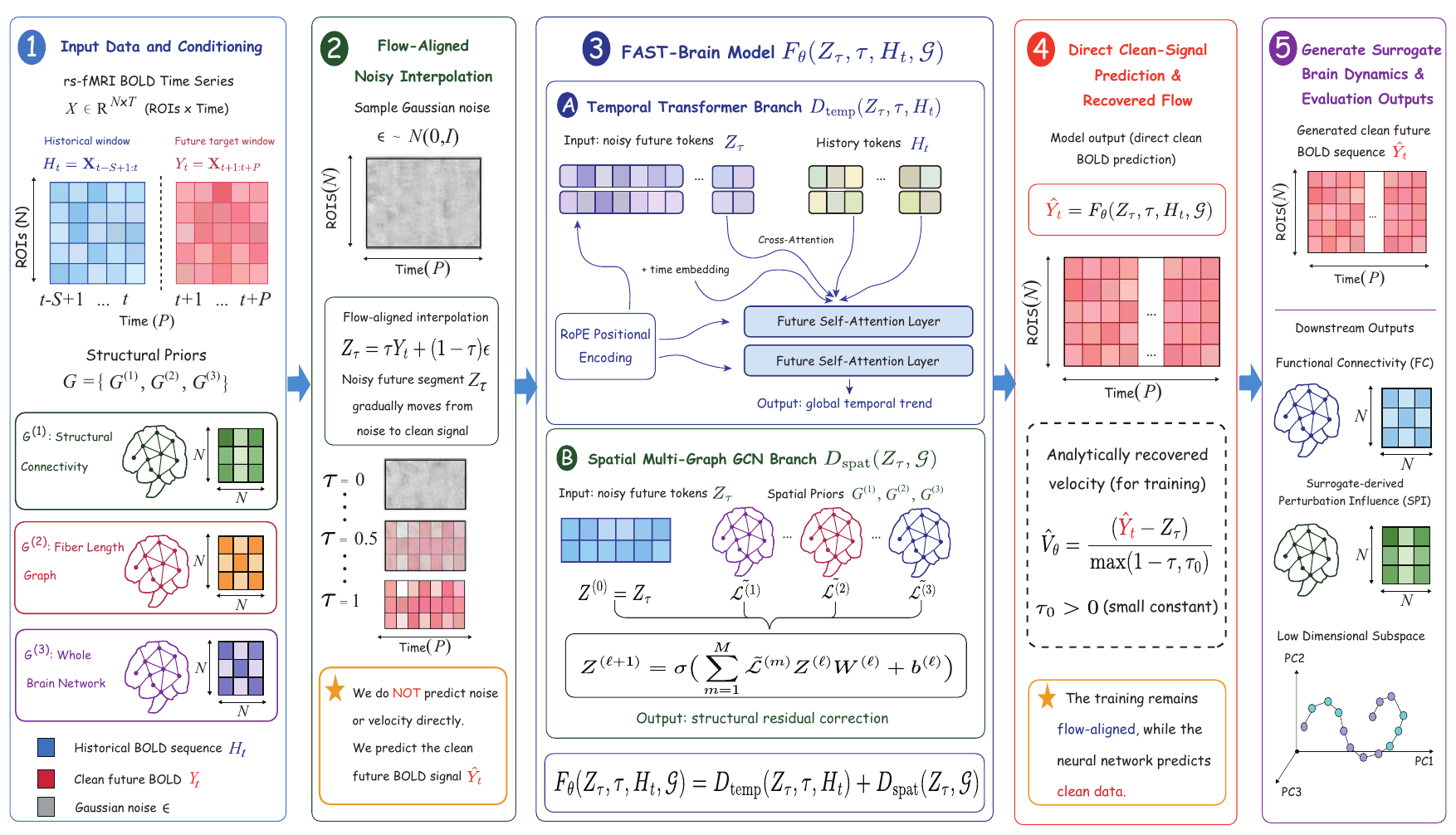}
    \vspace{-19pt}
    \caption{Model architecture of FAST-Brain. The framework integrates flow-aligned clean-signal prediction with a temporal Transformer and a spatial multi-graph GCN to generate surrogate BOLD dynamics for downstream connectivity and subspace analysis.}
    \vspace{-4pt}
    \label{fig:fast_brain_overview}
\end{figure}

To address all three limitations, we propose \textbf{FAST-Brain}, a \textbf{F}low-\textbf{A}ligned
\textbf{S}patio-\textbf{T}emporal surrogate \textbf{Brain} model. FAST-Brain integrates three
components (Figure~\ref{fig:fast_brain_overview}): (i) a \textbf{flow-aligned generative framework}
that directly predicts clean BOLD signals rather than noise or velocity; (ii) a \textbf{Transformer}
\citep{vaswani2017attention} that captures long-range temporal dependencies via history-conditioned
cross-attention; and (iii) a \textbf{graph convolutional network} \citep[GCN;][]{kipf2016semi,
geng2019spatiotemporal} that embeds structural connectivity, fiber lengths, and functional network
membership as spatial priors. Directly predicting the clean signal---rather than a noise or velocity
residual---enables the model to exploit the low-dimensional geometry of BOLD signals, improving
both learning efficiency and fidelity \citep{li2025back}.

We provide theoretical justification for our design. Under a low-dimensional subspace assumption
on fMRI signals, we prove that the Bayes-optimal denoiser depends on the noisy observation only
through its low-dimensional projection, and that the approximation error of our model scales with
the intrinsic dimension $d$ rather than the ambient ROI dimension $N$. In contrast, noise or
velocity-based prediction targets do not admit the same factorization, retaining dependence on
the full ambient space. 
Empirically, extensive experiments on synthetic datasets (RNN and SDDEs) and the real Human
Connectome Project (HCP) dataset demonstrate that FAST-Brain achieves state-of-the-art performance
in functional connectivity recovery, SPI estimation, and recovery of the intrinsic
low-dimensional signal geometry. On the HCP data, FAST-Brain raises functional connectivity correlation
from $0.726$ to $\mathbf{0.938}$ and reduces MAE from $0.199$ to $\mathbf{0.073}$ relative
to the respective strongest baselines.

\section{Related Work}

Our work intersects three lines of research: generative modeling for time series, direct
data prediction for image generation, and spatio-temporal modeling of brain dynamics. We review each in turn
and clarify how FAST-Brain advances beyond the current frontier.

\textbf{Generative Models for Time Series.}
Early conditional generative models for time-series relied on variational autoencoders
(VAEs) \citep{desai2021timevae, orlichenko2024demographic} and generative adversarial
networks (GANs) \citep{xu2020cot, tan2024brainfc}, which struggle with training stability
and mode diversity respectively. Diffusion-based models subsequently emerged as a more
principled alternative, achieving stable generation through iterative denoising
\citep{ho2020denoising, kong2020diffwave, tashiro2021csdi}; representative models include
SSSD \citep{alcaraz2022diffusion}, TimeDiff \citep{shen2023non}, and Diffusion-TS
\citep{yuan2024diffusion}. More recently, flow matching has reformulated generation as
learning a continuous velocity field that transports a simple prior to the data distribution
\citep{lipman2022flow, tamir2024conditional}, enabling faster and more direct deterministic
sampling than stochastic differential equation (SDE)-based diffusion. FlowTS extends this
paradigm to conditional time-series modeling \citep{hu2024flowts}, and recent work has
applied flow matching to fMRI generation \citep{tew2026functional}. Despite this progress,
all of these models---whether diffusion- or flow-based---parameterize noise or velocity
fields in the full ambient space, which, as we argue below, is ill-suited to signals
concentrated on a low-dimensional subspace. 

 \textbf{Direct Data Prediction for Image Generation.}
A growing body of work in image-domain generative modeling suggests that directly
predicting the clean data target, rather than noise or velocity, is more effective when
the signal lies near a low-dimensional manifold. Methods such as EDM
\citep{karras2022elucidating}, SiD2 \citep{hoogeboom2025simpler}, and JiT
\citep{li2025back} demonstrate consistent gains from this reparameterization in
high-dimensional image generation. The intuition is straightforward: the clean signal
inherits the low-dimensional structure of the data distribution, whereas noise and
velocity targets do not, making them harder to learn in high-dimensional ambient spaces.
Despite these advances in imaging, direct data prediction has received little attention
in structured time-series generation, and has not been explored for fMRI signals whose
dynamics are constrained by anatomical organization. FAST-Brain fills this gap.

\textbf{Spatio-temporal Modeling of Brain Dynamics.}
Classical whole-brain models include mechanistic simulators such as The Virtual Brain
\citep{sanz2013virtual} and dynamical systems frameworks \citep{breakspear2017dynamic},
as well as connectivity inference methods such as dynamic causal modeling (DCM)
\citep{friston2014dcm} and Granger causality (GC) \citep{granger1969investigating}.
While these approaches offer interpretability, they rely on predefined generative
assumptions or linear dependencies, limiting their ability to capture nonlinear
spatio-temporal dynamics at scale \citep{PhysRevE.97.052216}. Data-driven surrogate
brain models and digital twins have emerged as a more flexible alternative, learning
brain dynamics directly from observed fMRI data \citep{tu2019state, Lu2023, Luo2025}.
However, existing data-driven models either neglect anatomical spatial priors or fail
to exploit the intrinsic low-dimensional structure of fMRI signals---the two gaps that
FAST-Brain directly addresses through its graph-guided spatial decoder and flow-aligned
clean-signal prediction framework.

\section{Preliminaries}

We introduce the problem setup for rs-fMRI modeling and review flow matching,
the generative framework underlying FAST-Brain.

\textbf{Problem Setup.} We consider the problem of modeling rs-fMRI dynamics. For a given subject, the
observed data consist of BOLD time series recorded across $N$ pre-defined regions
of interest (ROIs) over $T$ acquisition time points, organized into a matrix
$X \in \mathbb{R}^{N \times T}$, where each row corresponds to one ROI and each
column to one time point.

Beyond the BOLD observations, we have access to three subject-level
anatomical information that summarize the brain's structural organization.
First, a structural connectivity matrix $SC \in \mathbb{R}^{N \times N}$, where
$SC_{ij}$ encodes the strength of the white-matter connection between ROIs $i$
and $j$. Second, a fiber length matrix $L \in \mathbb{R}^{N \times N}$, where
$L_{ij}$ is the average fiber length between the same pair. Third, a coarse
functional partition $c_i \in \{1, \ldots, C\}$ that assigns each ROI to one
of $C$ macroscopic functional networks. Both $SC$ and $L$ are normalized to
$[0,1]$. We treat $SC$, $L$, and $\{c_i\}$ collectively as time-invariant
\emph{spatial priors} that encode biophysically grounded constraints on
inter-regional interactions, and denote them by $\mathcal{G}$. How these
priors enter the model is described in Section~\ref{subsec:spatiotemporal}.

Our goal is to learn how BOLD activity evolves over time, conditional on its
recent history and these spatial priors. We adopt a sliding-window formulation:
at each time index $t$, the history window $H_t = X_{t-S+1:t} \in
\mathbb{R}^{N \times S}$ collects the past $S$ time points, and the future
target window $Y_t = X_{t+1:t+P} \in \mathbb{R}^{N \times P}$ contains the
next $P$ time points to be generated. Training the model reduces to learning
the conditional distribution $p(Y_t \mid H_t, \mathcal{G})$ over future BOLD
trajectories. Once trained, the model is rolled out autoregressively to produce
surrogate BOLD sequences of arbitrary length, from which downstream quantities
such as functional connectivity (FC) and SPI can be
derived.

\textbf{Flow Matching}
Flow matching \citep{lipman2022flow} learns a continuous-time vector field that
transports a simple prior distribution toward the target data distribution. We
briefly review the framework as it applies to our setting.

Let $Y_t \in \mathbb{R}^{N \times P}$ denote the future BOLD segment to be
generated, and let $\epsilon \sim \mathcal{N}(0, I)$ be a noise sample of the
same dimension drawn from a standard Gaussian prior. Flow matching defines a
transport process indexed by $\tau \in [0,1]$, connecting noise to data via a
linear interpolation
$
    Z_\tau = \tau Y_t + (1-\tau)\epsilon,
    \label{eq:interpolation}
$
so that $Z_\tau \sim p_{\mathrm{noise}}$ at $\tau = 0$ and $Z_\tau \sim
p_{\mathrm{data}}$ at $\tau = 1$ \citep{tu2019state, lipman2022flow,
hu2024flowts}. The corresponding flow velocity is
 $
    V = {dZ_\tau}/{d\tau} = Y_t - \epsilon.
$
Standard flow matching learns a parametric velocity field $V_\theta$ by
minimizing the mean squared error against this target velocity,
$
    \mathcal{L}_{\mathrm{FM}} =
    \mathbb{E}_{\tau,\,Y_t,\,\epsilon}
    \bigl[\|V_\theta(Z_\tau,\tau) - V\|^2\bigr].
    \label{eq:fm_loss}
$
At inference, a BOLD sequence is generated by integrating the learned ODE
$
    {dZ_\tau}/{d\tau} = V_\theta(Z_\tau,\tau),
     \ Z_0 \sim p_{\mathrm{noise}},
$ 
from $\tau=0$ to $1$ using a standard numerical solver, yielding a clean
sample $Z_1 \approx Y_t$.

In FAST-Brain, we build on this framework but replace velocity prediction with
direct clean-signal prediction, a design choice we motivate theoretically and
empirically in Section~\ref{sec:method}.

\section{Method}\label{sec:method}

We propose FAST-Brain, a flow-aligned spatio-temporal surrogate brain model for
rs-fMRI generation. The model has two core components: (i) a \emph{flow-aligned
direct data prediction} framework that predicts clean BOLD signals rather than
noise or velocity, and (ii) a spatio-temporal network architecture that combines
a history-conditioned Transformer for long-range temporal dependencies with a
GCN that injects anatomical spatial priors $\mathcal{G}$ as a residual
correction. Together, these components address the three limitations identified
in Section~\ref{sec:intro}. The complete pipeline is visualized in
Figure~\ref{fig:fast_brain_overview}.

Compared with step-by-step autoregressive models, our multi-step generation
design reduces exposure bias, in which small prediction errors compound over
the rollout horizon \citep{bengio2015scheduled}.

\subsection{Flow-Aligned Direct Data Prediction}
\label{subsec:prediction}

\paragraph{Motivation.}
Standard flow matching parameterizes a velocity field $V_\theta(Z_\tau, \tau)$
that lives in the full ambient $N \times P$ space. For rs-fMRI signals, which
are believed to concentrate near a low-dimensional subspace
\citep{gallego2017neural, chen2023score, pezon2024linking}, this is
inefficient: the velocity target contains components off the signal subspace
that carry no information about brain dynamics, making the network harder to
learn in high-dimensional ROI regimes. We therefore replace velocity prediction
with direct clean-signal prediction, building on analogous findings in
image-domain generative modeling \citep{karras2022elucidating, li2025back,
hoogeboom2025simpler}. We show in Section~\ref{subsec:theory} that this choice
is not merely heuristic: the Bayes-optimal denoiser inherits the same
low-dimensional structure as the signal itself.

\paragraph{Framework.}
Given the historical context $H_t$ and spatial priors $\mathcal{G}$, let
$Y_t \in \mathbb{R}^{N \times P}$ denote the clean future BOLD target. We
retain the flow matching interpolation path
$ 
    Z_\tau = \tau Y_t + (1-\tau)\epsilon, \quad \epsilon \sim \mathcal{N}(0,I),
$ 
with target velocity $V = Y_t - \epsilon$. Rather than predicting $V$ (or
$\epsilon$) directly, FAST-Brain parameterizes the clean signal,
$   \widehat{Y}_t = F_\theta(Z_\tau,\, \tau,\, H_t,\, \mathcal{G}),
$  
where the architecture of $F_\theta$ is detailed in
Section~\ref{subsec:spatiotemporal}. The velocity field is then recovered
analytically as
$ 
    \widehat{V}_\theta
    =  (\widehat{Y}_t - Z_\tau)/{\max(1-\tau,\,\tau_0)},
    \qquad \tau_0 = 0.05, 
$ 
where $\tau_0$ is a small clipping constant that prevents division by zero near
$\tau = 1$. Substituting into the flow matching objective yields the
FAST-Brain training loss,
\begin{equation}
    \mathcal{L}_{\mathrm{FAST}}
    = \mathbb{E}\!\left[
    \left\|
    \frac{F_\theta(Z_\tau,\tau,H_t,\mathcal{G}) - Z_\tau}{\max(1-\tau,\tau_0)}
    - (Y_t - \epsilon)
    \right\|_2^2
    \right].
    \label{eq:fast_loss}
\end{equation}
This formulation preserves the flow-aligned training objective while anchoring
all predictions in the data space. The training time $\tau$ follows a
logit-normal schedule, consistent with standard practice \citep{li2025back}.
Algorithms 1 and 2 summarize the full
training and sampling procedures.

\subsection{Spatio-Temporal Network Architecture}
\label{subsec:spatiotemporal}

We implement $F_\theta$ as the sum of two complementary decoders,
\begin{equation}
    F_\theta(Z_\tau,\tau,H_t,\mathcal{G})
    = D_{\mathrm{temp}}(Z_\tau,\tau,H_t)
    + D_{\mathrm{spat}}(Z_\tau,\mathcal{G}),
    \label{eq:decomposition}
\end{equation}
where $D_{\mathrm{temp}}$ captures history-dependent temporal evolution and
$D_{\mathrm{spat}}$ injects time-invariant anatomical structure as a residual
correction. The additive decomposition is intentional: it allows the two
branches to be trained in a curriculum fashion, as described below.

\paragraph{Temporal Decoder.}
$D_{\mathrm{temp}}(Z_\tau,\tau,H_t)$ models the large-scale temporal trend of
the future signal conditioned on the historical window, implemented as a
history-conditioned Transformer \citep{vaswani2017attention}. 
The future noisy signal $Z_\tau$, augmented with a time embedding of $\tau$,
serves as queries, while the historical window $H_t$ serves as keys and values
in a multi-head cross-attention layer. Rotary position embeddings
\citep{su2024roformer} are applied to the query and key representations before
attention is computed, encoding relative temporal position without requiring
absolute positional tokens. The cross-attention output passes through residual
connections, layer normalization, and two self-attention layers, allowing the
$P$ future time points to interact with one another across the forecast horizon.

\paragraph{Spatial Decoder.}
$D_{\mathrm{spat}}(Z_\tau,\mathcal{G})$ injects anatomical constraints as a
residual correction on top of the temporal trend. We encode the three spatial
priors as graph adjacency matrices,
\begin{equation}
    G^{(1)}_{ij} = SC_{ij}, \qquad
    G^{(2)}_{ij} = 1 - L_{ij}, \qquad
    G^{(3)}_{ij} = \mathbf{1}(c_i = c_j),
    \label{eq:graph_matrices}
\end{equation}
where $G^{(1)}$ captures direct anatomical coupling strength, $G^{(2)}$ assigns
larger weights to shorter fiber pathways (since $L$ is normalized to $[0,1]$),
and $G^{(3)}$ encodes shared functional-network membership. For each
$G^{(m)}$, we form the symmetrically normalized Laplacian
\begin{equation}
    \mathcal{L}^{(m)}
    = I - (D^{(m)})^{-1/2}\, G^{(m)}\, (D^{(m)})^{-1/2},
    \label{eq:laplacian}
\end{equation}
where $D^{(m)}$ is the corresponding degree matrix. To capture higher-order
structural interactions, we fuse the three Laplacians via a learnable
polynomial,
\begin{equation}
    \tilde{\mathcal{L}}^{(m)}
    = \sum_{k=0}^{K-1} \alpha_{m,k}\,(\mathcal{L}^{(m)})^k,
    \label{eq:poly_fusion}
\end{equation}
with learnable coefficients $\alpha_{m,k}$. Node features are then updated
through $L_g$ GCN layers,
\begin{equation}
    Z^{(\ell+1)}
    = \sigma\!\left(
    \sum_{m=1}^{M} \tilde{\mathcal{L}}^{(m)} Z^{(\ell)} W^{(\ell)} + b^{(\ell)}
    \right), \quad \ell = 0,\ldots,L_g-1,
    \label{eq:gcn_update}
\end{equation}
where $Z^{(0)} = Z_\tau \in \mathbb{R}^{N \times P}$ is the initialization,
and $W^{(\ell)} \in \mathbb{R}^{P_\ell \times P_{\ell+1}}$ is a feature
transformation shared across all nodes. We use $L_g = 2$ layers in all
experiments.

\paragraph{Curriculum initialization.}
The final projection layer of $D_{\mathrm{spat}}$ is zero-initialized, so the
spatial branch contributes nothing at the start of training. This implements a
lightweight curriculum: the model first learns the global temporal trend through
$D_{\mathrm{temp}}$, then progressively incorporates structural corrections from
$D_{\mathrm{spat}}$ as training proceeds. The additive structure in
Eq.~\eqref{eq:decomposition} ensures that $D_{\mathrm{temp}}$ is never
destabilized by the spatial branch during early optimization.

\begin{figure}[t]
\centering
\begin{minipage}[t]{0.48\linewidth}
\small
\hrule height 0.8pt
\vspace{4pt}
\noindent\textbf{Algorithm 1} Training   
\vspace{4pt}
\hrule height 0.4pt
\vspace{4pt}
\begin{algorithmic}[1]
\Require Dataset $\mathcal D$, priors $\mathcal G$, clip $\tau_0$, parameters $\theta$
\For{each mini-batch $(H_t,Y_t)\sim\mathcal D$}
\State Sample $\epsilon\sim\mathcal N(0,I)$ and $s$ from normal distribution; set $\tau=\sigma(s)$
\State $Z_\tau \gets \tau Y_t + (1-\tau)\epsilon$
\State $\widehat Y_t \gets F_\theta(Z_\tau,\tau,H_t,\mathcal G)$
\State $\widehat V_\theta \gets (\widehat Y_t-Z_\tau)/\max(1-\tau,\tau_0)$
\State $\mathcal L_{\mathrm{FAST}}\gets \|\widehat V_\theta-(Y_t-\epsilon)\|_2^2$
\EndFor
\State \Return trained $\theta^\star$
\end{algorithmic}
\vspace{4pt}
\hrule height 0.8pt
\end{minipage}\hfill
\begin{minipage}[t]{0.48\linewidth}
\small
\hrule height 0.8pt
\vspace{4pt}
\noindent\textbf{Algorithm 2} Sampling  
\vspace{4pt}
\hrule height 0.4pt
\vspace{4pt}
\begin{algorithmic}[1]
\Require History $H_t$, priors $\mathcal G$, steps $K=20$, clip $\tau_0$
\State Sample $Z_0\sim\mathcal N(0,I)$ in $\mathbb{R}^{N\times P}$,set $\Delta\tau\gets 1/K$
\For{$i=0,\ldots,K-1$}
\State $\tau_i\gets i/K$
\State $\widehat Y_i\gets F_\theta(Z_i,\tau_i,H_t,\mathcal G)$
\State $\widehat V_i\gets (\widehat Y_i-Z_i)/\max(1-\tau_i,\tau_0)$
\State $Z_{i+1}\gets Z_i+\Delta\tau\,\widehat V_i$
\EndFor
\State \Return $\widehat Y_t\gets Z_K$
\end{algorithmic}
\vspace{4pt}
\hrule height 0.8pt
\end{minipage}
\vspace{2mm}
\label{fig:fast_algorithms}
\end{figure}

\subsection{Theoretical Guarantees on Low-Dimensional Subspaces}
\label{subsec:theory}

\paragraph{Setup and motivation.}
Resting-state fMRI signals are believed to concentrate near a low-dimensional
subspace shaped by brain anatomy \citep{gallego2017neural, chen2023score,
pezon2024linking}. A natural question is whether our direct clean-signal
prediction strategy can exploit this structure more efficiently than
noise- or velocity-based parameterizations. We answer this affirmatively:
under a low-dimensional subspace assumption, the Bayes-optimal denoiser
depends on the noisy observation \emph{only} through its intrinsic projection,
and the approximation error of our model scales with the intrinsic dimension
$d$ rather than the ambient ROI dimension $N$. Noise- and velocity-based
targets, by contrast, do not admit the same factorization.

\paragraph{Reduction to a weighted prediction loss.}
We first observe that, over $\tau \in [\tau_0, 1-\tau_0]$, the FAST-Brain
loss~\eqref{eq:fast_loss} is equivalent to a weighted squared loss for
predicting $Y_t$ directly. Substituting $Z_\tau = \tau Y_t + (1-\tau)\epsilon$
and rearranging,
$     \mathcal{L}_{\mathrm{FAST}}
    = \mathbb{E}\!\left[
     {(1-\tau)^{-2}}
    \bigl\|F_\theta(Z_\tau,\tau,H_t,\mathcal{G}) - Y_t\bigr\|_2^2
    \right].
$
Therefore, for each fixed $(Z_\tau{=}Z,\,\tau,\,H_t{=}h,\,\mathcal{G}{=}g)$,
the Bayes-optimal predictor under this objective is the conditional mean
\begin{equation}
    F^*(Z,\tau,h,g)
    := \mathbb{E}\bigl[Y_t \mid Z_\tau=Z,\, H_t=h,\, \mathcal{G}=g\bigr].
    \label{eq:bayes_opt}
\end{equation}
The main result below characterizes the structure of $F^*$ and shows that
a neural network approximating it incurs an error governed by $d$, not $N$.

We impose three regularity conditions.
\begin{assumption}[Low-dimensional signal subspace]
\label{ass:subspace}
The future signal $Y_t$ lies almost surely in a $d$-dimensional subspace,
with $d \ll N$. Specifically, $Y_t = AS_t$ a.s., where $S_t \in
\mathbb{R}^{d \times P}$ is a latent coefficient matrix and $A \in
\mathbb{R}^{N \times d}$ is an orthonormal basis satisfying $A^\top A = I_d$.
We write $Z_\tau' = A^\top Z_\tau \in \mathbb{R}^{d \times P}$ for the
projected noisy state, with conditional density $p_{\tau|h,g}'$.
\end{assumption}

\begin{assumption}[Regularity of the latent denoiser]
\label{ass:regularity}
The conditioning spaces $\mathcal{H} \subset \mathbb{R}^{m_h}$ and
$\mathcal{G} \subset \mathbb{R}^{m_g}$ are compact. The latent denoiser
$
    f(Z',\tau,h,g)
    := \frac{1}{\tau}\bigl(Z' + (1-\tau)^2 \nabla\log p_{\tau|h,g}'(Z')\bigr)
$   
is jointly continuous on $[-R,R]^{d\times P} \times [\tau_0,1-\tau_0]
\times \mathcal{H} \times \mathcal{G}$ for every $R > 0$.
\end{assumption}

\begin{assumption}[Linear growth and uniform integrability]
\label{ass:growth}
There exist constants $C_0, C_1 > 0$ such that $\|f(Z',\tau,h,g)\|_F \le
C_0 + C_1\|Z'\|_F$ for all $(Z',\tau,h,g)$. Moreover, $Z_\tau'$ is uniformly
square-integrable conditional on $(H_t{=}h, \mathcal{G}{=}g)$:
$   \lim_{R\to\infty}
    \sup_{\tau,h,g}
    \mathbb{E}\!\left[
    \|Z_\tau'\|_F^2\,
    \mathbf{1}\{\|Z_\tau'\|_F > R\}
    \mid H_t=h,\,\mathcal{G}=g
    \right] = 0.
$ 
\end{assumption}

Assumption~\ref{ass:subspace} formalizes the widely held belief that rs-fMRI
signals have intrinsic low-dimensional structure
\citep{gallego2017neural,pezon2024linking}. Assumption~\ref{ass:regularity}
requires joint continuity of the latent denoiser; smoothness of the noisy
latent density follows from Gaussian convolution. Assumption~\ref{ass:growth}
imposes mild linear growth and tail conditions on the latent denoiser.

\begin{theorem}[Low-dimensional approximation of the Bayes-optimal denoiser]
\label{thm:lowddenoiser}
Fix $\tau_0 \in (0,1/2)$. Under Assumptions~\ref{ass:subspace}--\ref{ass:growth},
the Bayes-optimal denoiser $F^*$ defined in \eqref{eq:bayes_opt} admits the
low-dimensional factorization
\begin{equation}
    F^*(Z,\tau,h,g) = A\, f(A^\top Z,\tau,h,g),
    \label{eq:factorization}
\end{equation}
where $f$ is the latent denoiser defined in Assumption \ref{ass:regularity}.
Moreover, for every $\varepsilon > 0$, there exists a ReLU feedforward neural
network $f_\theta : \mathbb{R}^{d\times P} \times [\tau_0,1-\tau_0] \times
\mathcal{H} \times \mathcal{G} \to \mathbb{R}^{d\times P}$ such that the
induced predictor $F_{A,\theta}(Z,\tau,h,g) := A f_\theta(A^\top Z,\tau,h,g)$
satisfies
\begin{equation}
    \sup_{\substack{\tau \in [\tau_0,1-\tau_0] \\ h \in \mathcal{H},\;
    g \in \mathcal{G}}}
    \bigl\|F_{A,\theta}(\cdot,\tau,h,g)
    - F^*(\cdot,\tau,h,g)\bigr\|_{L^2(P_{\tau|h,g})}
    \le \sqrt{dP+1}\;\varepsilon.
    \label{eq:approx_bound}
\end{equation}
\end{theorem}

\paragraph{Interpretation.}
Equation~\eqref{eq:factorization} shows that $F^*$ depends on the noisy
observation $Z_\tau \in \mathbb{R}^{N\times P}$ \emph{only} through its
low-dimensional projection $A^\top Z_\tau \in \mathbb{R}^{d\times P}$.
The approximation bound~\eqref{eq:approx_bound} then follows from the
factorization of $F^*$ through the low-dimensional projection.
Thus, the approximation problem is governed by the intrinsic dimension $d$
rather than the ambient ROI dimension $N$.

By contrast, noise- and velocity-based prediction targets involve the full
noisy signal $Z_\tau \in \mathbb{R}^{N\times P}$ and do not admit the
factorization~\eqref{eq:factorization}. Concretely, the induced velocity
$v^*(Z_\tau,\tau,h,g) = (F^*(Z_\tau,\tau,h,g) - Z_\tau)/(1-\tau)$ contains
the term $-Z_\tau/(1-\tau)$, which lives in the full ambient space and
carries no signal about brain dynamics. Approximating such targets therefore
retains dependence on $N$, making learning harder as the number of ROIs
grows.

In summary, Theorem~\ref{thm:lowddenoiser} provides a rigorous justification
for our design choice: direct clean-signal prediction reduces the effective
dimensionality of the learning problem from the ambient ROI dimension $N$ to
the intrinsic signal dimension $d$, enabling more efficient learning
of brain dynamics.

\section{Experiments}

\textbf{Data}. We evaluate FAST-Brain on two synthetic datasets, RNN~\citep{Luo2025} and SDDEs~\citep{sanz2013virtual}, and one real rs-fMRI dataset, HCP~\citep{van2013wu}. For the RNN dataset, structural connectivity is defined by fiber length and large-scale network information. This dataset contains 20 ROIs and 8000 time points. The SDDEs dataset is generated by stochastic delay differential equations with stochastic delay effects and more complex temporal coupling, and contains 66 ROIs and 3000 time points. The HCP dataset contains 360 ROIs and 4280 time points. Details of data generation are provided in the Appendix. For all datasets, the first 80\% of the sequence samples are used for training.

\textbf{Baseline Algorithms}. We compare FAST-Brain against both fMRI-domain and general time-series prediction  algorithms. The fMRI-domain baselines include NPI \citep{Luo2025} and DSFM \citep{tew2026functional}, while the general time-series baselines include TimeDiff \citep{shen2023non}, SSSD \citep{alcaraz2022diffusion}, DiffusionTS \citep{yuan2024diffusion}, FlowTS \citep{hu2024flowts}, and the latent state-space model LS4 \citep{zhou2023deep}.

\textbf{Metrics}. Functional connectivity (FC) measures statistical dependence between ROI activities over time. We estimate FC between two regions $i$ and $j$ by
\[
\mathrm{FC}_{ij} =
\frac{\sum_{t=1}^{T} (X_{i,t} - \bar{X}_i)(X_{j,t} - \bar{X}_j)}
{\sqrt{\sum_{t=1}^{T} (X_{i,t} - \bar{X}_i)^2}\sqrt{\sum_{t=1}^{T} (X_{j,t} - \bar{X}_j)^2}},
\]
where $\bar{X}_i$ is the temporal mean of ROI $i$. For all datasets, we evaluate FC using Pearson correlation and mean absolute error (MAE) from the generated and reference FC matrices.

SPI measures the local sensitivity of predicted activity in one region to activity in another. Motivated by the perturbation-based analysis of \citet{Luo2025}, we define
$
\mathrm{SPI}_{i\to j}
=
\mathbb{E}_{t}\!\left[
\lim_{\delta\to0}
\frac{\widehat{Y}^{(i,+\delta)}_{t+T,j} - \widehat{Y}_{t+T,j}}{\delta}
\right],
$
where the perturbation is applied to ROI $i$ in the final history frame. We compute one-step sensitivities of the clean-signal predictor via automatic differentiation, yielding an SPI matrix in $\mathbb{R}^{N\times N}$. 
For SPI evaluation on synthetic data with known ground-truth directed interactions, we report correlation, AUROC \citep{fawcett2006introduction}, and AUPRC \citep{saito2015precision}. Correlation measures the agreement between estimated SPI and ground-truth interaction strengths, AUROC evaluates how well the estimated SPI scores distinguish true directed links from non-links, and AUPRC summarizes precision-recall performance, which is particularly informative for sparse connectivity graphs. 

Beyond FC and SPI, we evaluate the conditional generation quality of each algorithm using CRPS~\citep{matheson1976scoring}, QICE~\citep{han2022card}, ProbCorr~\citep{ni2021sig}, and Conditional FID~\citep{yue2022ts2vec}. Here, CRPS assesses the discrepancy between the predicted conditional distribution and the observed future signal. QICE evaluates calibration by measuring whether the true observations fall into generated quantile intervals at the expected frequencies. ProbCorr measures how well generated samples preserve the correlation structure of real signals. Conditional FID measures distributional similarity in a learned time-series representation space by extracting features from real and generated sequences using a pretrained TS2Vec encoder and comparing their conditional feature distributions. Their detailed definitions are given in the Appendix. We also report metrics that measure the unconditional generation quality, including Discriminative Score, Predictive Score, Context-FID, and Correlational Score \citep{yoon2019time, yuan2024diffusion}. 

For sequence-based evaluations, each trained generator is conditioned on an initial historical window and used to predict the next 24 time steps (i.e., $P=24$). The generated sequence is then fed back autoregressively until the generated sequence reaches the same length as the test sequence. These sequence-based metrics are computed from the generated sequences. For each trained model, we repeat generation 10 times and report the mean and standard deviation (SD; $\mathrm{ddof}=0$) of each metric.

\begingroup
\setlength{\intextsep}{12pt}
\begin{table}[H]
\centering
\scriptsize

\caption{Functional connectivity recovery across three datasets.}
\vspace{-5pt} 
\label{tab:fc_main}
\ExpTableFont

\setlength{\tabcolsep}{2pt}
\renewcommand{\arraystretch}{1.2}
\begin{tabular}{@{}llcccccccc@{}}
\toprule
Dataset & Metric & NPI & TimeDiff & SSSD & DiffusionTS & FlowTS & DSFM & LS4 & FAST-Brain \\
\midrule
\multirow{2}{*}{RNN}
& Correlation$\uparrow$ & 0.547$_{\scriptstyle \pm 0.073}$ & 0.431$_{\scriptstyle \pm 0.003}$ & 0.822$_{\scriptstyle \pm 0.010}$ & 0.824$_{\scriptstyle \pm 0.006}$ & 0.846$_{\scriptstyle \pm 0.005}$ & 0.695$_{\scriptstyle \pm 0.008}$ & 0.869$_{\scriptstyle \pm 0.012}$ & 0.938$_{\scriptstyle \pm 0.006}$ \\
& MAE$\downarrow$       & 0.322$_{\scriptstyle \pm 0.075}$ & 0.364$_{\scriptstyle \pm 0.003}$ & 0.050$_{\scriptstyle \pm 0.002}$ & 0.062$_{\scriptstyle \pm 0.001}$ & 0.059$_{\scriptstyle \pm 0.001}$ & 0.087$_{\scriptstyle \pm 0.004}$ & 0.055$_{\scriptstyle \pm 0.003}$ & 0.028$_{\scriptstyle \pm 0.002}$ \\
\midrule
\multirow{2}{*}{SDDEs}
& Correlation$\uparrow$ & 0.789$_{\scriptstyle \pm 0.009}$ & 0.821$_{\scriptstyle \pm 0.000}$ & 0.890$_{\scriptstyle \pm 0.001}$ & 0.886$_{\scriptstyle \pm 0.018}$ & 0.830$_{\scriptstyle \pm 0.008}$ & 0.642$_{\scriptstyle \pm 0.008}$ & 0.790$_{\scriptstyle \pm 0.000}$ & 0.975$_{\scriptstyle \pm 0.009}$ \\
& MAE$\downarrow$       & 0.332$_{\scriptstyle \pm 0.006}$ & 0.311$_{\scriptstyle \pm 0.000}$ & 0.170$_{\scriptstyle \pm 0.001}$ & 0.169$_{\scriptstyle \pm 0.012}$ & 0.275$_{\scriptstyle \pm 0.003}$ & 0.304$_{\scriptstyle \pm 0.004}$ & 0.318$_{\scriptstyle \pm 0.000}$ & 0.115$_{\scriptstyle \pm 0.023}$ \\
\midrule
\multirow{2}{*}{HCP}
& Correlation$\uparrow$ & 0.634$_{\scriptstyle \pm 0.074}$ & 0.344$_{\scriptstyle \pm 0.002}$ & 0.186$_{\scriptstyle \pm 0.002}$ & 0.474$_{\scriptstyle \pm 0.004}$ & 0.726$_{\scriptstyle \pm 0.008}$ & 0.692$_{\scriptstyle \pm 0.035}$ & 0.676$_{\scriptstyle \pm 0.042}$ & 0.938$_{\scriptstyle \pm 0.005}$ \\
& MAE$\downarrow$       & 0.212$_{\scriptstyle \pm 0.030}$ & 0.386$_{\scriptstyle \pm 0.001}$ & 0.402$_{\scriptstyle \pm 0.000}$ & 0.201$_{\scriptstyle \pm 0.027}$ & 0.280$_{\scriptstyle \pm 0.005}$ & 0.199$_{\scriptstyle \pm 0.032}$ & 0.252$_{\scriptstyle \pm 0.008}$ & 0.073$_{\scriptstyle \pm 0.008}$ \\
\bottomrule
\end{tabular}
\end{table}
\endgroup
\vspace{-0.2cm}
\textbf{Functional Connectivity.} Table~\ref{tab:fc_main} shows that FAST-Brain consistently achieves the highest FC correlation and lowest MAE relative to the ground-truth signals across all three datasets. This indicates that the generated signals preserve both the global connectivity topology, as measured by correlation, and the absolute FC strengths, as measured by MAE. The improvement is especially pronounced on the real HCP dataset, where FAST-Brain increases the best baseline correlation from 0.726 to 0.938 and reduces MAE from 0.199 to 0.073. These suggest that combining flow-aligned temporal generation with spatial brain constraints yields more faithful surrogate brain models.

\begin{wraptable}{r}{0.55\textwidth}
\vspace{-12pt} 
\centering
\caption{SPI analysis on RNN dataset.}
\label{tab:ec_main}
\setlength{\tabcolsep}{4.5pt}
\footnotesize
\begin{tabular}{lccc}
\toprule
Method & Correlation$\uparrow$ & AUROC$\uparrow$ & AUPRC$\uparrow$ \\
\midrule
NPI         & 0.065$_{\scriptstyle \pm 0.063}$  & 0.466$_{\scriptstyle \pm 0.061}$ & 0.093$_{\scriptstyle \pm 0.020}$ \\
TimeDiff    & 0.013$_{\scriptstyle \pm 0.046}$  & 0.487$_{\scriptstyle \pm 0.053}$ & 0.102$_{\scriptstyle \pm 0.023}$ \\
SSSD        & 0.828$_{\scriptstyle \pm 0.005}$  & 0.851$_{\scriptstyle \pm 0.005}$ & 0.736$_{\scriptstyle \pm 0.005}$ \\
DiffusionTS & 0.716$_{\scriptstyle \pm 0.021}$  & 0.848$_{\scriptstyle \pm 0.030}$ & 0.600$_{\scriptstyle \pm 0.049}$ \\
FlowTS      & 0.851$_{\scriptstyle \pm 0.004}$  & 0.916$_{\scriptstyle \pm 0.006}$ & 0.810$_{\scriptstyle \pm 0.007}$ \\
DSFM        & -0.006$_{\scriptstyle \pm 0.022}$ & 0.462$_{\scriptstyle \pm 0.028}$ & 0.093$_{\scriptstyle \pm 0.010}$ \\
LS4         & 0.898$_{\scriptstyle \pm 0.004}$ & 0.990$_{\scriptstyle \pm 0.003}$ & 0.936$_{\scriptstyle \pm 0.010}$ \\
FAST-Brain  & 0.956$_{\scriptstyle \pm 0.000}$  & 0.999$_{\scriptstyle \pm 0.000}$ & 0.995$_{\scriptstyle \pm 0.000}$ \\
\bottomrule
\end{tabular}
\end{wraptable}

\textbf{SPI.} Table~\ref{tab:ec_main} shows that FAST-Brain recovers the ground-truth interactions more accurately than all baselines on synthetic RNN data. It achieves the highest correlation, AUROC, and AUPRC, outperforming the strongest baseline LS4. NPI performs less competitively in this setting, likely because its autoregressive one-step prediction approach makes it difficult to capture long-range dynamical variations accumulated over extended horizons. The near-perfect AUROC and AUPRC scores indicate that FAST-Brain can accurately identify directed interactions induced by the underlying dynamical system. Evaluation against known ground-truth interactions is restricted to synthetic data.

\textbf{Conditional Generation.}
We further compare the conditional generation quality of FAST-Brain against general time-series models on the RNN dataset. As shown in Table~\ref{tab:cond_main}, FAST-Brain achieves the best performance on all four metrics, obtaining the lowest CRPS, QICE, ProbCorr, and Conditional FID among all compared methods. These results indicate that FAST-Brain produces more accurate probabilistic forecasts, better-calibrated predictive intervals, more faithful correlation structure, and closer alignment with the real conditional distribution under the same historical context.
\begin{table}[H]
\vspace{0pt}
\centering
\caption{Conditional generation quality on the RNN dataset.}
\vspace{-2pt}
\label{tab:cond_main}
\ExpTableFont
\renewcommand{\arraystretch}{1.2}
\scriptsize
\setlength{\tabcolsep}{2pt}

\begin{tabular}{@{}lcccccccc@{}}
\toprule
Metric & NPI & TimeDiff & SSSD & DiffusionTS & FlowTS & DSFM & LS4 & FAST-Brain \\
\midrule
CRPS$\downarrow$            & 0.248$_{\scriptstyle \pm 0.002}$ & 0.213$_{\scriptstyle \pm 0.000}$ & 0.217$_{\scriptstyle \pm 0.000}$ & 0.158$_{\scriptstyle \pm 0.000}$ & 0.168$_{\scriptstyle \pm 0.000}$ & 0.176$_{\scriptstyle \pm 0.000}$ & 0.179$_{\scriptstyle \pm 0.000}$ & 0.154$_{\scriptstyle \pm 0.000}$ \\
QICE$\downarrow$            & 0.160$_{\scriptstyle \pm 0.000}$ & 0.140$_{\scriptstyle \pm 0.001}$ & 0.144$_{\scriptstyle \pm 0.000}$ & 0.011$_{\scriptstyle \pm 0.000}$ & 0.069$_{\scriptstyle \pm 0.000}$ & 0.007$_{\scriptstyle \pm 0.000}$ & 0.078$_{\scriptstyle \pm 0.000}$ & 0.005$_{\scriptstyle \pm 0.000}$ \\
ProbCorr$\downarrow$        & 0.317$_{\scriptstyle \pm 0.004}$ & 0.314$_{\scriptstyle \pm 0.000}$ & 0.264$_{\scriptstyle \pm 0.000}$ & 0.266$_{\scriptstyle \pm 0.000}$ & 0.280$_{\scriptstyle \pm 0.000}$ & 0.305$_{\scriptstyle \pm 0.001}$ & 0.266$_{\scriptstyle \pm 0.001}$ & 0.261$_{\scriptstyle \pm 0.000}$ \\
Conditional FID$\downarrow$ & 3.516$_{\scriptstyle \pm 0.040}$ & 8.384$_{\scriptstyle \pm 1.238}$ & 2.522$_{\scriptstyle \pm 0.012}$ & 1.916$_{\scriptstyle \pm 0.002}$ & 2.020$_{\scriptstyle \pm 0.002}$ & 2.812$_{\scriptstyle \pm 0.387}$ & 1.946$_{\scriptstyle \pm 0.005}$ & 1.856$_{\scriptstyle \pm 0.003}$ \\
\bottomrule
\end{tabular}
\end{table}
\begin{wraptable}[10]{R}{0.66\textwidth}
\vspace{-13pt}
\centering
\caption{Ablation study on functional connectivity recovery.}
\label{tab:ablation_main}
\setlength{\tabcolsep}{4pt}
\footnotesize
\begin{tabular}{llcccc}
\toprule
Dataset & Metric & FAST-Brain & w/o GCN & $\epsilon$-pred. & $v$-pred.  \\
\midrule
\multirow{2}{*}{RNN}
& Corr$\uparrow$ & 0.938$_{\scriptstyle \pm 0.006}$ & 0.932$_{\scriptstyle \pm 0.006}$ & 0.928$_{\scriptstyle \pm 0.011}$ & 0.919$_{\scriptstyle \pm 0.014}$ \\
& MAE$\downarrow$ & 0.028$_{\scriptstyle \pm 0.002}$ & 0.029$_{\scriptstyle \pm 0.002}$ & 0.028$_{\scriptstyle \pm 0.002}$ & 0.032$_{\scriptstyle \pm 0.003}$ \\
\midrule
\multirow{2}{*}{SDDEs}
& Corr$\uparrow$ & 0.975$_{\scriptstyle \pm 0.009}$ & 0.827$_{\scriptstyle \pm 0.002}$ & 0.862$_{\scriptstyle \pm 0.004}$ & 0.822$_{\scriptstyle \pm 0.004}$ \\
& MAE$\downarrow$ & 0.115$_{\scriptstyle \pm 0.023}$ &  0.305$_{\scriptstyle \pm 0.001}$  & 0.130$_{\scriptstyle \pm 0.004}$ & 0.292$_{\scriptstyle \pm 0.003}$ \\
\midrule
\multirow{2}{*}{HCP}
& Corr$\uparrow$ & 0.938$_{\scriptstyle \pm 0.005}$ & 0.867$_{\scriptstyle \pm 0.015}$ & 0.799$_{\scriptstyle \pm 0.009}$ & 0.856$_{\scriptstyle \pm 0.014}$ \\
& MAE$\downarrow$ & 0.073$_{\scriptstyle \pm 0.008}$ & 0.133$_{\scriptstyle \pm 0.017}$ & 0.313$_{\scriptstyle \pm 0.004}$ & 0.209$_{\scriptstyle \pm 0.017}$ \\
\bottomrule
\end{tabular}
\end{wraptable}

\textbf{Ablation Study.} As discussed earlier, FAST-Brain consists of three main ingredients: a GCN for encoding spatial constraints, a Transformer for capturing long-range temporal dependencies, and a direct prediction strategy that predicts the clean BOLD signal. To assess the contributions of the spatial constraints and the direct clean-signal prediction objective, we conduct ablation studies on all three datasets. We compare the full FAST-Brain model with three variants: one that removes the GCN module, denoted by \emph{w/o GCN}; one that replaces direct clean-signal prediction with the conventional DDPM noise-prediction objective, denoted by \emph{\(\epsilon\)-prediction}; and one that replaces it with standard velocity prediction in flow matching~\citep{lipman2022flow}, denoted by \emph{\(v\)-prediction}.

Table~\ref{tab:ablation_main} reports functional connectivity recovery for these variants. The full FAST-Brain model consistently achieves the best performance across all three datasets, indicating that both the direct clean-signal prediction objective and the graph-based spatial priors contribute to accurate FC recovery. The benefit of directly predicting clean BOLD signals is especially evident on the higher-dimensional SDDEs and HCP datasets, where replacing this objective with \(\epsilon\)-prediction or \(v\)-prediction leads to noticeably worse performance. This observation is consistent with our theoretical result that the Bayes-optimal denoiser depends on the noisy observation only through its intrinsic low-dimensional projection. Removing the GCN module further degrades performance, particularly on SDDEs and HCP, demonstrating that anatomical and network-level priors provide useful structural guidance. Together, these results show that FAST-Brain benefits from leveraging both low-dimensional clean-signal structure and biologically informed spatial constraints.

\begin{wrapfigure}{r}{0.62\textwidth}
\vspace{-8pt}
\centering
\includegraphics[width=0.60\textwidth]{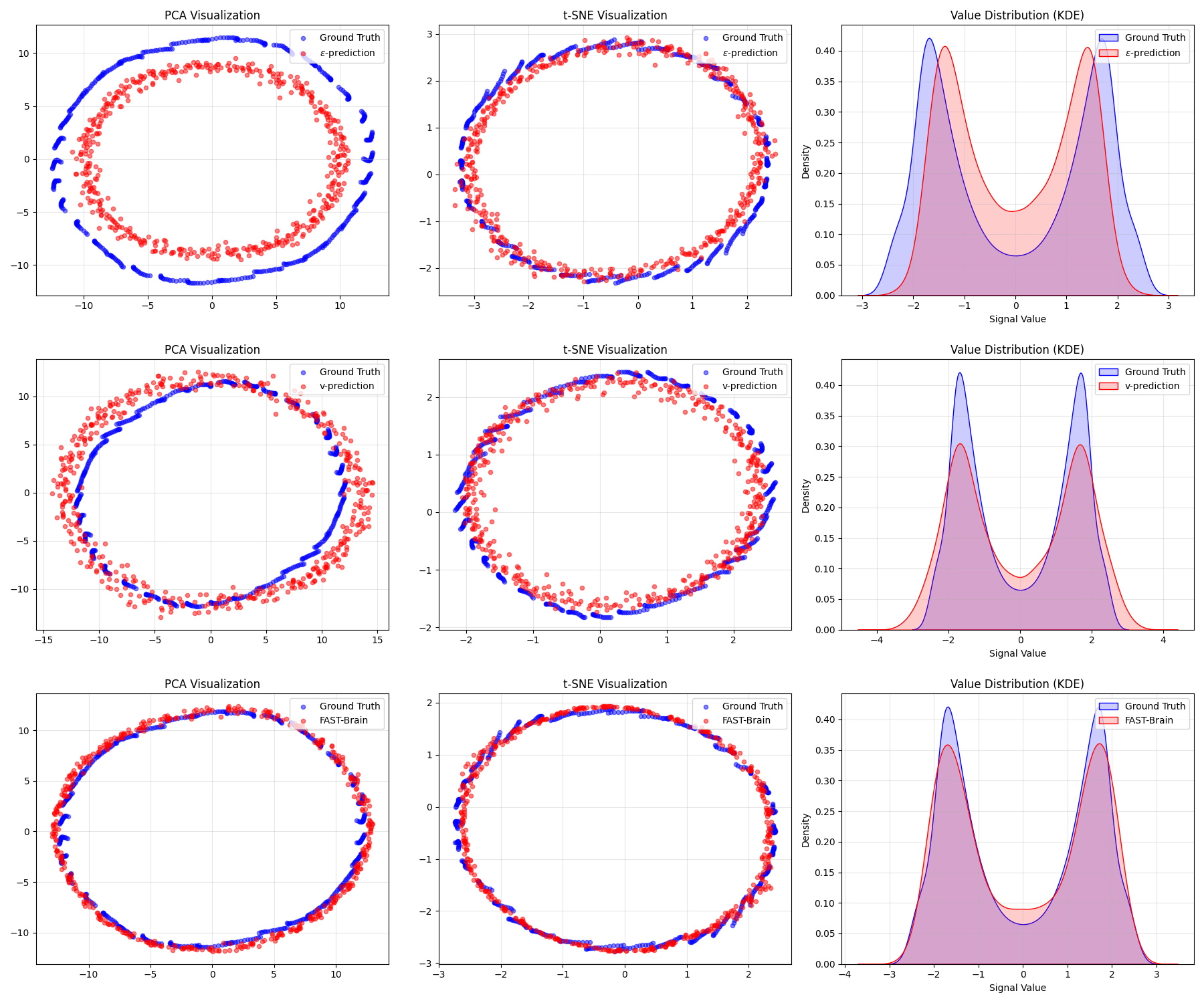}
\vspace{-6pt}
\caption{Latent-space and value-distribution analyses on the SDDEs dataset.}
\label{fig:latent_sddes}
\vspace{-4pt}
\end{wrapfigure}

\textbf{Latent-space analysis.}
We further visualize the generated SDDEs samples using PCA~\citep{abdi2010principal} and t-SNE~\citep{van2008visualizing} to assess whether different parameterizations preserve the intrinsic signal geometry. Specifically, we compare FAST-Brain, which predicts the clean signal, with \(\epsilon\)-prediction, which predicts noise, and \(v\)-prediction, which predicts velocity. We also visualize the kernel density estimate (KDE) of the marginal signal-value distribution and compare it with that of the ground-truth samples. Figure~\ref{fig:latent_sddes} shows that $\epsilon$-prediction produces a smaller and misaligned ring in the PCA space, while $v$-prediction produces more dispersed samples and an over-smoothed marginal distribution. In contrast, FAST-Brain closely matches the ground-truth circular manifold in both PCA and t-SNE and better preserves the bimodal KDE shape. This suggests that FAST-Brain recovers both the latent trajectory structure and the signal-value distribution more faithfully. These empirical results support our theoretical finding that direct clean-signal prediction can better recover the low-dimensional signal than noise or velocity prediction.

\section{Conclusion}

We presented FAST-Brain, a flow-aligned spatio-temporal surrogate brain model for resting-state fMRI generation. FAST-Brain directly predicts clean BOLD signals while incorporating graph-based anatomical priors and handling long-range temporal dependencies, enabling more faithful recovery of functional connectivity, SPI, and low-dimensional signal geometry. For neuroimaging, FAST-Brain supports in silico perturbation experiments that generate hypotheses about interregional BOLD responses for validation against independent stimulation data.

In addition, our theory characterizes the low-dimensional structure of the Bayes-optimal denoiser and the approximation error of the proposed model, but a finite-sample statistical learning analysis remains open. Deriving nonasymptotic bounds on statistical estimation errors would further clarify when and why direct clean-signal prediction improves high-dimensional spatio-temporal generation.

\begin{ack}
Dr. Zhu's work was supported in part by the Gillings Innovation Laboratory on generative AI; the National Institute on Aging (NIA) of the National Institutes of Health (NIH) under grants 1R01AG085581 and RF1AG082938; the National Institute of Mental Health (NIMH) under grant 1R01MH136055; and the NIH under grants R01AR082684 and 1OT2OD038045-01. Dr. Zhang's research was partially supported by NSF grants DMS-2152289, TI-2449855, and DMS-2610629, as well as BSF grant 2024055. The content of this paper is solely the responsibility of the authors and does not necessarily represent the official views of the supporting institutions.
\end{ack}

\bibliographystyle{plainnat}
\bibliography{references}

@article{park2013structural,
  title={Structural and functional brain networks: from connections to cognition},
  author={Park, Hae-Jeong and Friston, Karl},
  journal={Science},
  volume={342},
  number={6158},
  pages={1238411},
  year={2013},
  publisher={American Association for the Advancement of Science}
}

@inproceedings{han2024brainode,
  title={{BrainODE}: Dynamic brain signal analysis via graph-aided neural ordinary
         differential equations},
  author={Han, Kaiqiao and Yang, Yi and Huang, Zijie and Kan, Xuan and Guo, Ying
          and Yang, Yang and He, Lifang and Zhan, Liang and Sun, Yizhou
          and Wang, Wei and Yang, Carl},
  booktitle={2024 IEEE EMBS International Conference on Biomedical and
             Health Informatics (BHI)},
  year={2024},
  organization={IEEE}
}

@article{wein2022forecasting,
  title={Forecasting brain activity based on models of spatiotemporal brain dynamics:
         a comparison of graph neural network architectures},
  author={Wein, Simon and Sch{\"u}ller, A. and Tom{\'e}, Ana M. and Malloni, Wilhelm M.
          and Greenlee, Mark W. and Lang, Elmar W.},
  journal={Network Neuroscience},
  volume={6},
  number={3},
  pages={665--701},
  year={2022},
  doi={10.1162/netn_a_00252},
  publisher={MIT Press}
}

@article{kang2026statistical,
  title={Statistical Opportunities in Neuroimaging},
  author={Kang, Jian and Nichols, Thomas and Li, Lexin and Lindquist, Martin A and Zhu, Hongtu},
  journal={arXiv preprint arXiv:2602.12974},
  year={2026}
}

@article{zhu2014fusing,
  title={Fusing DTI and fMRI data: a survey of methods and applications},
  author={Zhu, Dajiang and Zhang, Tuo and Jiang, Xi and Hu, Xintao and Chen, Hanbo and Yang, Ning and Lv, Jinglei and Han, Junwei and Guo, Lei and Liu, Tianming},
  journal={NeuroImage},
  volume={102},
  pages={184--191},
  year={2014},
  publisher={Elsevier}
}

@article{lindquist2025statistics,
  title={Statistics, Data Science, and the Connectome},
  author={Lindquist, Martin A and Caffo, Brian and Kang, Jian and Simpson, Sean and Pinto, Marco and Ting, Chee-Ming and Shojaie, Ali and Guindani, Michele and Ombao, Hernando},
  journal={Statistics and Data Science in Imaging},
  volume={2},
  number={1},
  pages={2569894},
  year={2025},
  publisher={Taylor \& Francis}
}

@article{zhu2023statistical,
  title={Statistical learning methods for neuroimaging data analysis with applications},
  author={Zhu, Hongtu and Li, Tengfei and Zhao, Bingxin},
  journal={Annual review of biomedical data science},
  volume={6},
  number={1},
  pages={73--104},
  year={2023},
  publisher={Annual Reviews}
}

@article{breakspear2017dynamic,
  title={Dynamic models of large-scale brain activity},
  author={Breakspear, Michael},
  journal={Nature neuroscience},
  volume={20},
  number={3},
  pages={340--352},
  year={2017},
  publisher={Nature Publishing Group US New York}
}

@article{kong2020diffwave,
  title={Diffwave: A versatile diffusion model for audio synthesis},
  author={Kong, Zhifeng and Ping, Wei and Huang, Jiaji and Zhao, Kexin and Catanzaro, Bryan},
  journal={arXiv preprint arXiv:2009.09761},
  year={2020}
}

@article{tashiro2021csdi,
  title={Csdi: Conditional score-based diffusion models for probabilistic time series imputation},
  author={Tashiro, Yusuke and Song, Jiaming and Song, Yang and Ermon, Stefano},
  journal={Advances in neural information processing systems},
  volume={34},
  pages={24804--24816},
  year={2021}
}

@article{xu2020cot,
  title={Cot-gan: Generating sequential data via causal optimal transport},
  author={Xu, Tianlin and Wenliang, Li Kevin and Munn, Michael and Acciaio, Beatrice},
  journal={Advances in neural information processing systems},
  volume={33},
  pages={8798--8809},
  year={2020}
}

@article{sanz2013virtual,
  title={The Virtual Brain: a simulator of primate brain network dynamics},
  author={Sanz Leon, Paula and Knock, Stuart A and Woodman, M Marmaduke and Domide, Lia and Mersmann, Jochen and McIntosh, Anthony R and Jirsa, Viktor},
  journal={Frontiers in neuroinformatics},
  volume={7},
  pages={10},
  year={2013},
  publisher={Frontiers Media SA}
}

@inproceedings{tamir2024conditional,
  title={Conditional flow matching for time series modelling},
  author={Tamir, Ella and Laabid, Najwa and Heinonen, Markus and Garg, Vikas and Solin, Arno},
  booktitle={ICML 2024 Workshop on Structured Probabilistic Inference and Generative Modeling},
  year={2024}
}

@inproceedings{geng2019spatiotemporal,
  title={Spatiotemporal multi-graph convolution network for ride-hailing demand forecasting},
  author={Geng, Xu and Li, Yaguang and Wang, Leye and Zhang, Lingyu and Yang, Qiang and Ye, Jieping and Liu, Yan},
  booktitle={Proceedings of the AAAI conference on artificial intelligence},
  volume={33},
  number={01},
  pages={3656--3663},
  year={2019}
}

@article{tu2019state,
  title={A state-space model for inferring effective connectivity of latent neural dynamics from simultaneous EEG/fMRI},
  author={Tu, Tao and Paisley, John and Haufe, Stefan and Sajda, Paul},
  journal={Advances in Neural Information Processing Systems},
  volume={32},
  year={2019}
}

@article{lipman2022flow,
  title={Flow matching for generative modeling},
  author={Lipman, Yaron and Chen, Ricky TQ and Ben-Hamu, Heli and Nickel, Maximilian and Le, Matt},
  journal={arXiv preprint arXiv:2210.02747},
  year={2022}
}

@article{granger1969investigating,
  title={Investigating causal relations by econometric models and cross-spectral methods},
  author={Granger, Clive WJ},
  journal={Econometrica: journal of the Econometric Society},
  pages={424--438},
  year={1969},
  publisher={JSTOR}
}

@article{friston2014dcm,
  title={A DCM for resting state fMRI},
  author={Friston, Karl J and Kahan, Joshua and Biswal, Bharat and Razi, Adeel},
  journal={Neuroimage},
  volume={94},
  pages={396--407},
  year={2014},
  publisher={Elsevier}
}

@article{PhysRevE.97.052216,
  title = {Causal inference in nonlinear systems: Granger causality versus time-delayed mutual information},
  author = {Li, Songting and Xiao, Yanyang and Zhou, Douglas and Cai, David},
  journal = {Phys. Rev. E},
  volume = {97},
  issue = {5},
  pages = {052216},
  numpages = {9},
  year = {2018},
  month = {May},
  publisher = {American Physical Society},
  doi = {10.1103/PhysRevE.97.052216},
  url = {https://link.aps.org/doi/10.1103/PhysRevE.97.052216}
}

@article{vincent2010stacked,
  title={Stacked denoising autoencoders: Learning useful representations in a deep network with a local denoising criterion.},
  author={Vincent, Pascal and Larochelle, Hugo and Lajoie, Isabelle and Bengio, Yoshua and Manzagol, Pierre-Antoine},
  journal={Journal of machine learning research},
  volume={11},
  number={12},
  year={2010}
}

@article{hagmann2008mapping,
  title={Mapping the structural core of human cerebral cortex},
  author={Hagmann, Patric and Cammoun, Leila and Gigandet, Xavier and Meuli, Reto and Honey, Christopher J and Wedeen, Van J and Sporns, Olaf},
  journal={PLoS biology},
  volume={6},
  number={7},
  pages={e159},
  year={2008},
  publisher={Public Library of Science San Francisco, USA}
}

@article{honey2009predicting,
  title={Predicting human resting-state functional connectivity from structural connectivity},
  author={Honey, Christopher J and Sporns, Olaf and Cammoun, Leila and Gigandet, Xavier and Thiran, Jean-Philippe and Meuli, Reto and Hagmann, Patric},
  journal={Proceedings of the National Academy of Sciences},
  volume={106},
  number={6},
  pages={2035--2040},
  year={2009},
  publisher={National Academy of Sciences}
}

@article{orlichenko2024demographic,
  title={A Demographic-Conditioned Variational Autoencoder for fMRI Distribution Sampling and Removal of Confounds},
  author={Orlichenko, Anton and Qu, Gang and Zhou, Ziyu and Liu, Anqi and Deng, Hong-Wen and Ding, Zhengming and Stephen, Julia M. and Wilson, Tony W. and Calhoun, Vince D. and Wang, Yu-Ping},
  journal={bioRxiv},
  year={2024},
  doi={10.1101/2024.05.16.594528}
}

@inproceedings{feng2026towards,
  title={Towards Interpretable Visual Decoding with Attention to Brain Representations},
  author={Pinyuan Feng and Hossein Adeli and Wenxuan Guo and Fan L. Cheng and Ethan Hwang and Nikolaus Kriegeskorte},
  booktitle={The Fourteenth International Conference on Learning Representations},
  year={2026},
  url={https://openreview.net/forum?id=YWlYITAhMC}
}

@inproceedings{
tew2026functional,
title={Functional {MRI} Time Series Generation via Wavelet-Based Image Transform and Spectral Flow Matching for Brain Disorder Identification},
author={Hwa Hui Tew and Junn Yong Loo and Leong Fang Yu and Julia K. Lau and Ding Fan and Hernando Ombao and Raphael CW Phan and Chee Pin Tan and Chee-Ming Ting},
booktitle={The Fourteenth International Conference on Learning Representations},
year={2026},
url={https://openreview.net/forum?id=Dgphd9qizu}
}

@inproceedings{tew2025t2idiff,
  title={T2I-Diff: fMRI Signal Generation via Time-Frequency Image Transform and Classifier-Free Denoising Diffusion Models},
  author={Tew, Hwa Hui and Loo, Junn Yong and Tan, Yee-Fan and Tang, Xinyu and Ombao, Hernando and Noman, Fuad and Phan, Rapha{\"e}l C.-W. and Ting, Chee-Ming},
  booktitle={Medical Image Computing and Computer Assisted Intervention -- MICCAI 2025},
  pages={640--650},
  year={2025},
  publisher={Springer}
}

@inproceedings{tan2024brainfc,
  title={Brainfc-cgan: A conditional generative adversarial network for brain functional connectivity augmentation and aging synthesis},
  author={Tan, Yee-Fan and Loo, Junn-Yong and Ting, Chee-Ming and Noman, Fuad and Phan, Rapha{\"e}l C-W and Ombao, Hernando},
  booktitle={ICASSP 2024-2024 IEEE International Conference on Acoustics, Speech and Signal Processing (ICASSP)},
  pages={1511--1515},
  year={2024},
  organization={IEEE}
}

@inproceedings{hoogeboom2025simpler,
  title={Simpler Diffusion: 1.5 FID on ImageNet512 with pixel-space diffusion},
  author={Hoogeboom, Emiel and Mensink, Thomas and Heek, Jonathan and Lamerigts, Kay and Gao, Ruiqi and Salimans, Tim},
  booktitle={Proceedings of the Computer Vision and Pattern Recognition Conference},
  pages={18062--18071},
  year={2025}
}

@article{karras2022elucidating,
  title={Elucidating the design space of diffusion-based generative models},
  author={Karras, Tero and Aittala, Miika and Aila, Timo and Laine, Samuli},
  journal={Advances in neural information processing systems},
  volume={35},
  pages={26565--26577},
  year={2022}
}

@article{hu2024flowts,
  title={FlowTS: Time Series Generation via Rectified Flow},
  author={Hu, Yang and Wang, Xiao and Ding, Zezhen and Wu, Lirong and Zhang, Huatian and Li, Stan Z and Wang, Sheng and Zhang, Jiheng and Li, Ziyun and Chen, Tianlong},
  journal={arXiv preprint arXiv:2411.07506},
  year={2024}
}

@article{yuan2024diffusion,
  title={Diffusion-ts: Interpretable diffusion for general time series generation},
  author={Yuan, Xinyu and Qiao, Yan},
  journal={arXiv preprint arXiv:2403.01742},
  year={2024}
}

@inproceedings{shen2023non,
  title={Non-autoregressive conditional diffusion models for time series prediction},
  author={Shen, Lifeng and Kwok, James},
  booktitle={International Conference on Machine Learning},
  pages={31016--31029},
  year={2023},
  organization={PMLR}
}

@article{alcaraz2022diffusion,
  title={Diffusion-based time series imputation and forecasting with structured state space models},
  author={Alcaraz, Juan Miguel Lopez and Strodthoff, Nils},
  journal={arXiv preprint arXiv:2208.09399},
  year={2022}
}

@article{li2025back,
  title={Back to basics: Let denoising generative models denoise},
  author={Li, Tianhong and He, Kaiming},
  journal={arXiv preprint arXiv:2511.13720},
  year={2025}
}

@article{ho2020denoising,
  title={Denoising diffusion probabilistic models},
  author={Ho, Jonathan and Jain, Ajay and Abbeel, Pieter},
  journal={Advances in neural information processing systems},
  volume={33},
  pages={6840--6851},
  year={2020}
}

@article{kipf2016semi,
  title={Semi-supervised classification with graph convolutional networks},
  author={Kipf, Thomas N and Welling, Max},
  journal={arXiv preprint arXiv:1609.02907},
  year={2016}
}

@article{vaswani2017attention,
  title={Attention is all you need},
  author={Vaswani, Ashish and Shazeer, Noam and Parmar, Niki and Uszkoreit, Jakob and Jones, Llion and Gomez, Aidan N and Kaiser, {\L}ukasz and Polosukhin, Illia},
  journal={Advances in neural information processing systems},
  volume={30},
  year={2017}
}

@article{gallego2017neural,
  title={Neural manifolds for the control of movement},
  author={Gallego, Juan A and Perich, Matthew G and Miller, Lee E and Solla, Sara A},
  journal={Neuron},
  volume={94},
  number={5},
  pages={978--984},
  year={2017},
  publisher={Elsevier}
}

@article{desai2021timevae,
  title={Timevae: A variational auto-encoder for multivariate time series generation},
  author={Desai, Abhyuday and Freeman, Cynthia and Wang, Zuhui and Beaver, Ian},
  journal={arXiv preprint arXiv:2111.08095},
  year={2021}
}

@article{deco2013resting,
  title={Resting-state functional connectivity emerges from structurally and dynamically shaped slow linear fluctuations},
  author={Deco, Gustavo and Ponce-Alvarez, Adri{\'a}n and Mantini, Dante and Romani, Gian Luca and Hagmann, Patric and Corbetta, Maurizio},
  journal={Journal of Neuroscience},
  volume={33},
  number={27},
  pages={11239--11252},
  year={2013},
  publisher={Society for Neuroscience}
}

@article{bengio2015scheduled,
  title={Scheduled sampling for sequence prediction with recurrent neural networks},
  author={Bengio, Samy and Vinyals, Oriol and Jaitly, Navdeep and Shazeer, Noam},
  journal={Advances in neural information processing systems},
  volume={28},
  year={2015}
}

@article{abdi2010principal,
  title={Principal component analysis},
  author={Abdi, Herv{\'e} and Williams, Lynne J},
  journal={Wiley interdisciplinary reviews: computational statistics},
  volume={2},
  number={4},
  pages={433--459},
  year={2010},
  publisher={Wiley Online Library}
}

@article{van2008visualizing,
  title={Visualizing data using t-SNE.},
  author={Van der Maaten, Laurens and Hinton, Geoffrey},
  journal={Journal of machine learning research},
  volume={9},
  number={11},
  year={2008}
}

@article{yoon2019time,
  title={Time-series generative adversarial networks},
  author={Yoon, Jinsung and Jarrett, Daniel and Van der Schaar, Mihaela},
  journal={Advances in neural information processing systems},
  volume={32},
  year={2019}
}

@inproceedings{yue2022ts2vec,
  title={Ts2vec: Towards universal representation of time series},
  author={Yue, Zhihan and Wang, Yujing and Duan, Juanyong and Yang, Tianmeng and Huang, Congrui and Tong, Yunhai and Xu, Bixiong},
  booktitle={Proceedings of the AAAI conference on artificial intelligence},
  volume={36},
  number={8},
  pages={8980--8987},
  year={2022}
}

@inproceedings{ni2021sig,
  title={Sig-Wasserstein GANs for time series generation},
  author={Ni, Hao and Szpruch, Lukasz and Sabate-Vidales, Marc and Xiao, Baoren and Wiese, Magnus and Liao, Shujian},
  booktitle={Proceedings of the Second ACM International Conference on AI in Finance},
  pages={1--8},
  year={2021}
}

@article{han2022card,
  title={Card: Classification and regression diffusion models},
  author={Han, Xizewen and Zheng, Huangjie and Zhou, Mingyuan},
  journal={Advances in Neural Information Processing Systems},
  volume={35},
  pages={18100--18115},
  year={2022}
}

@article{matheson1976scoring,
  title={Scoring rules for continuous probability distributions},
  author={Matheson, James E and Winkler, Robert L},
  journal={Management science},
  volume={22},
  number={10},
  pages={1087--1096},
  year={1976},
  publisher={INFORMS}
}

@article{saito2015precision,
  title={The precision-recall plot is more informative than the ROC plot when evaluating binary classifiers on imbalanced datasets},
  author={Saito, Takaya and Rehmsmeier, Marc},
  journal={PloS one},
  volume={10},
  number={3},
  pages={e0118432},
  year={2015},
  publisher={Public Library of Science}
}

@article{fawcett2006introduction,
  title={An introduction to ROC analysis},
  author={Fawcett, Tom},
  journal={Pattern recognition letters},
  volume={27},
  number={8},
  pages={861--874},
  year={2006},
  publisher={Elsevier}
}

@article{van2013wu,
  title={The WU-Minn human connectome project: an overview},
  author={Van Essen, David C and Smith, Stephen M and Barch, Deanna M and Behrens, Timothy EJ and Yacoub, Essa and Ugurbil, Kamil and Wu-Minn HCP Consortium and others},
  journal={Neuroimage},
  volume={80},
  pages={62--79},
  year={2013},
  publisher={Elsevier}
}

@article{efron2011tweedie,
  title={Tweedie’s formula and selection bias},
  author={Efron, Bradley},
  journal={Journal of the American Statistical Association},
  volume={106},
  number={496},
  pages={1602--1614},
  year={2011},
  publisher={Taylor \& Francis}
}

@inproceedings{chen2023score,
  title={Score approximation, estimation and distribution recovery of diffusion models on low-dimensional data},
  author={Chen, Minshuo and Huang, Kaixuan and Zhao, Tuo and Wang, Mengdi},
  booktitle={International Conference on Machine Learning},
  pages={4672--4712},
  year={2023},
  organization={PMLR}
}

@article{su2024roformer,
  title={Roformer: Enhanced transformer with rotary position embedding},
  author={Su, Jianlin and Ahmed, Murtadha and Lu, Yu and Pan, Shengfeng and Bo, Wen and Liu, Yunfeng},
  journal={Neurocomputing},
  volume={568},
  pages={127063},
  year={2024},
  publisher={Elsevier}
}

@article{pezon2024linking,
  title={Linking neural manifolds to circuit structure in recurrent networks},
  author={Pezon, Louis and Schmutz, Valentin and Gerstner, Wulfram},
  journal={bioRxiv},
  doi={10.1101/2024.02.28.582565},
  year={2024},
  publisher={Cold Spring Harbor Laboratory}
}

@article{perich2020inferring,
  title={Inferring brain-wide interactions using data-constrained recurrent neural network models},
  author={Perich, Matthew G and Arlt, Charlotte and Soares, Sofia and Young, Megan E and Mosher, Clayton P and Minxha, Juri and Carter, Eugene and Rutishauser, Ueli and Rudebeck, Peter H and Harvey, Christopher D and others},
  journal={BioRxiv},
  doi={10.1101/2020.12.18.423348},
  year={2020},
  publisher={Cold Spring Harbor Laboratory}
}

@article{deco2021revisiting,
  title={Revisiting the global workspace orchestrating the hierarchical organization of the human brain},
  author={Deco, Gustavo and Vidaurre, Diego and Kringelbach, Morten L},
  journal={Nature human behaviour},
  volume={5},
  number={4},
  pages={497--511},
  year={2021},
  publisher={Nature Publishing Group UK London}
}

@article{van2010exploring,
  title={Exploring the brain network: a review on resting-state fMRI functional connectivity},
  author={van den Heuvel, Martijn P and Hulshoff Pol, Hilleke E},
  journal={European neuropsychopharmacology},
  volume={20},
  number={8},
  pages={519--534},
  year={2010},
  publisher={Elsevier}
}

@article{calhoun2014chronnectome,
  title={The chronnectome: time-varying connectivity networks as the next frontier in fMRI data discovery},
  author={Calhoun, Vince D and Miller, Robyn and Pearlson, Godfrey and Adal{\i}, Tulay},
  journal={Neuron},
  volume={84},
  number={2},
  pages={262--274},
  year={2014},
  publisher={Elsevier}
}

@article{Biswal1995,
  title={Functional connectivity in the motor cortex of resting human brain using echo-planar MRI},
  author={Biswal, Bharat and Yetkin, F. Zerrin and Haughton, Victor M. and Hyde, James S.},
  journal={Magnetic Resonance in Medicine},
  volume={34},
  number={4},
  pages={537--541},
  year={1995},
  doi={10.1002/mrm.1910340409}
}

@article{Friston2011,
  title={Functional and Effective Connectivity: A Review},
  author={Friston, Karl J.},
  journal={Brain Connectivity},
  volume={1},
  number={1},
  pages={13--36},
  year={2011},
  doi={10.1089/brain.2011.0008}
}

@article{Luo2025,
  title={Mapping effective connectivity by virtually perturbing a surrogate brain},
  author={Luo, Zixiang and Peng, Kaining and Liang, Zhichao and Cai, Shengyuan and Xu, Chenyu and Li, Dan and Hu, Yu and Zhou, Changsong and Liu, Quanying},
  journal={Nature Methods},
  volume={22},
  number={6},
  pages={1376--1385},
  year={2025},
  doi={10.1038/s41592-025-02654-x}
}

@article{Lu2023,
  title={Digital Twin Brain: a simulation and assimilation platform for whole human brain},
  author={Lu, Wenlian and Zeng, Longbin and Du, Xin and Zhang, Wenyong and Xiang, Shitong and Wang, Huarui and Wang, Jiexiang and Ji, Mingda and Hou, Yubo and Wang, Minglong and Liu, Yuhao and Chen, Zhongyu and Zheng, Qibao and Xu, Ningsheng and Feng, Jianfeng},
  journal={arXiv preprint arXiv:2308.01241},
  year={2023},
  doi={10.48550/arXiv.2308.01241},
  url={https://arxiv.org/abs/2308.01241}
}

@article{Thomas2022,
  title={Self-Supervised Learning of Brain Dynamics from Broad Neuroimaging Data},
  author={Thomas, Armin W. and Re, Christopher and Poldrack, Russell A.},
  journal={arXiv preprint arXiv:2206.11417},
  year={2022},
  doi={10.48550/arXiv.2206.11417},
  url={https://arxiv.org/abs/2206.11417}
}

@article{Dong2024,
  title={Brain-JEPA: Brain Dynamics Foundation Model with Gradient Positioning and Spatiotemporal Masking},
  author={Dong, Zijian and Li, Ruilin and Wu, Yilei and Nguyen, Thuan Tinh and Chong, Joanna Su Xian and Ji, Fang and Tong, Nathanael Ren Jie and Chen, Christopher Li Hsian and Zhou, Juan Helen},
  journal={arXiv preprint arXiv:2409.19407},
  year={2024},
  doi={10.48550/arXiv.2409.19407},
  url={https://arxiv.org/abs/2409.19407}
}

@inproceedings{zhou2023deep,
  title = {Deep Latent State Space Models for Time-Series Generation},
  author = {Zhou, Linqi and Poli, Michael and Xu, Winnie and Massaroli, Stefano and Ermon, Stefano},
  booktitle = {Proceedings of the 40th International Conference on Machine Learning},
  pages = {42625--42643},
  year = {2023},
  volume = {202},
  series = {Proceedings of Machine Learning Research},
  publisher = {PMLR},
  url = {https://proceedings.mlr.press/v202/zhou23i.html}
}

@article{li2026connectome,
  title = {Connectome-based spatial statistics enabling large-scale population analyses of human connectome across cohorts},
  author = {Li, Tengfei and Wang, Xifeng and Cole, Martin and Sun, Zehui and Jiang, Zhiwen and Qian, Xinjie and Gao, Shan and Luo, Tianyou and Descoteaux, Maxime and Stein, Jason L. and Wang, Xiao and Nichols, Thomas E. and Zhang, Heping and Zhang, Zhengwu and Zhu, Hongtu},
  journal = {bioRxiv},
  year = {2026},
  doi = {10.64898/2026.04.09.717492},
  url = {https://doi.org/10.64898/2026.04.09.717492}
}

@article{zhang2018mapping,
  title = {Mapping population-based structural connectomes},
  author = {Zhang, Zhengwu and Descoteaux, Maxime and Zhang, Jingwen and Girard, Gabriel and Chamberland, Maxime and Dunson, David and Srivastava, Anuj and Zhu, Hongtu},
  journal = {NeuroImage},
  volume = {172},
  pages = {130--145},
  year = {2018},
  doi = {10.1016/j.neuroimage.2017.12.064},
  url = {https://doi.org/10.1016/j.neuroimage.2017.12.064}
}

\appendix

\newpage

\appendix

\section{Proof of Theorem 1}

\begin{proof}

We divide the proof into four steps.

\medskip
\noindent
\textbf{Step 1: Structural form of the Bayes-optimal denoiser.}
Fix $(\tau,h,g)\in[\tau_0, 1-\tau_0]\times\mathcal H\times\mathcal G$. By Assumption 1, we have
\[
Y_t=A S_t \quad \text{a.s.},
\]
where $S_t\in\mathbb R^{d\times P}$ and $A^\top A=I_d$. Recall that
\[
Z_\tau=\tau Y_t+(1-\tau)\varepsilon,
\]
where $\varepsilon\in\mathbb R^{N\times P}$ has i.i.d. standard Gaussian entries and is independent of $(Y_t,H_t,\mathcal G)$.

Let
\[
P_A(Z):=AA^\top Z,\qquad P_\perp(Z):=(I_N-AA^\top)Z
\]
denote the orthogonal projections onto $\mathrm{span}(A)$ and its orthogonal complement, respectively  \citep{chen2023score}. Since $Y_t=AS_t$, we have
\[
P_A(Y_t)=Y_t,\qquad P_\perp(Y_t)=0.
\]
Therefore,
\begin{align*}
P_A(Z_\tau)
&=P_A\bigl(\tau Y_t+(1-\tau)\varepsilon\bigr) \\
&=\tau P_A(Y_t)+(1-\tau)P_A(\varepsilon) \\
&=\tau Y_t+(1-\tau)P_A(\varepsilon),
\end{align*}
and similarly
\begin{align*}
P_\perp(Z_\tau)
&=P_\perp\bigl(\tau Y_t+(1-\tau)\varepsilon\bigr) \\
&=\tau P_\perp(Y_t)+(1-\tau)P_\perp(\varepsilon) \\
&=(1-\tau)P_\perp(\varepsilon).
\end{align*}

Because $\varepsilon$ is isotropic Gaussian, the two Gaussian components $P_A(\varepsilon)$ and $P_\perp(\varepsilon)$ are independent. Hence, conditionally on $(H_t=h,\mathcal G=g)$, the conditional density of $Z_\tau$ factorizes into the density of its on-subspace component and the density of its orthogonal component. Writing
\[
Z':=A^\top Z\in\mathbb R^{d\times P},
\]
we may write
\[
p_{\tau\mid h,g}(Z)=p_{\tau\mid h,g}'(A^\top Z)\,p_\tau^\perp(P_\perp Z),
\]
where $p_\tau^\perp$ denotes the Gaussian density of the orthogonal component.

Taking logarithms and differentiating with respect to $Z$, we obtain
\begin{align*}
\nabla_Z\log p_{\tau\mid h,g}(Z)
&=\nabla_Z\log p_{\tau\mid h,g}'(A^\top Z)+\nabla_Z\log p_\tau^\perp(P_\perp Z).
\end{align*}
For the first term, by the chain rule,
\[
\nabla_Z\log p_{\tau\mid h,g}'(A^\top Z)=A\nabla \log p_{\tau\mid h,g}'(A^\top Z).
\]
For the second term, since $P_\perp(Z_\tau)$ is centered Gaussian with covariance $(1-\tau)^2I$ on the orthogonal component, its score is
\[
\nabla_Z\log p_\tau^\perp(P_\perp Z)
=-\frac{1}{(1-\tau)^2}P_\perp(Z)
=-\frac{1}{(1-\tau)^2}(Z-AA^\top Z).
\]
Therefore,
\begin{align*}
\nabla_Z\log p_{\tau\mid h,g}(Z)
&=A\nabla \log p_{\tau\mid h,g}'(A^\top Z)
-\frac{1}{(1-\tau)^2}(Z-AA^\top Z).
\end{align*}

On the other hand, under the Gaussian perturbation model
\[
Z_\tau\mid Y_t\sim\mathcal N(\tau Y_t,(1-\tau)^2I),
\]
Tweedie's formula \citep{efron2011tweedie} yields
\begin{align*}
\nabla_Z\log p_{\tau\mid h,g}(Z)
&=-\frac{1}{(1-\tau)^2}\Bigl(Z-\tau \mathbb E[Y_t\mid Z_\tau=Z,H_t=h,\mathcal G=g]\Bigr) \\
&=-\frac{1}{(1-\tau)^2}\bigl(Z-\tau F^*(Z,\tau,h,g)\bigr).
\end{align*}
Rearranging gives
\begin{align*}
F^*(Z,\tau,h,g)
&=\frac{1}{\tau}\Bigl(Z+(1-\tau)^2\nabla_Z\log p_{\tau\mid h,g}(Z)\Bigr).
\end{align*}
Substituting the score decomposition above, we obtain
\begin{align*}
F^*(Z,\tau,h,g)
&=\frac{1}{\tau}\Bigl(Z+(1-\tau)^2A\nabla \log p_{\tau\mid h,g}'(A^\top Z)-(Z-AA^\top Z)\Bigr) \\
&=\frac{1}{\tau}\Bigl(AA^\top Z+(1-\tau)^2A\nabla \log p_{\tau\mid h,g}'(A^\top Z)\Bigr) \\
&=A\left[\frac{1}{\tau}\Bigl(A^\top Z+(1-\tau)^2\nabla \log p_{\tau\mid h,g}'(A^\top Z)\Bigr)\right].
\end{align*}
Now define
\[
f(Z',\tau,h,g):=\frac{1}{\tau}\Bigl(Z'+(1-\tau)^2\nabla \log p_{\tau\mid h,g}'(Z')\Bigr).
\]
Then
\[
F^*(Z,\tau,h,g)=Af(A^\top Z,\tau,h,g).
\]
This proves the representation formula and shows that the Bayes-optimal denoiser depends only on the low-dimensional projection $A^\top Z$.

\medskip
\noindent
\textbf{Step 2: Reduction of the approximation problem to the latent space.}
Let
\[
f_\theta:\mathbb R^{d\times P}\times[\tau_0, 1-\tau_0]\times\mathcal H\times\mathcal G\to\mathbb R^{d\times P}
\]
be any measurable map, and define the induced predictor
\[
F_{A,\theta}(Z,\tau,h,g):=Af_\theta(A^\top Z,\tau,h,g).
\]
Since $A^\top A=I_d$, the map $M\mapsto AM$ preserves the Frobenius norm. Indeed, for any $M\in\mathbb R^{d\times P}$,
\begin{align*}
\|AM\|_F^2
&=\operatorname{tr}\bigl((AM)^\top(AM)\bigr) \\
&=\operatorname{tr}(M^\top A^\top AM) \\
&=\operatorname{tr}(M^\top M) \\
&=\|M\|_F^2.
\end{align*}
Applying this identity with $M=f_\theta(A^\top Z,\tau,h,g)-f(A^\top Z,\tau,h,g)$, we get
\begin{align*}
\|F_{A,\theta}(Z,\tau,h,g)-F^*(Z,\tau,h,g)\|_F^2
&=\|A(f_\theta(A^\top Z,\tau,h,g)-f(A^\top Z,\tau,h,g))\|_F^2 \\
&=\|f_\theta(A^\top Z,\tau,h,g)-f(A^\top Z,\tau,h,g)\|_F^2.
\end{align*}

Now let $Z_\tau\sim P_{\tau\mid h,g}$ and define $Z_\tau':=A^\top Z_\tau$. Then $Z_\tau'\sim P_{\tau\mid h,g}'$. Hence
\begin{align*}
\|F_{A,\theta}(\cdot,\tau,h,g)-F^*(\cdot,\tau,h,g)\|_{L^2(P_{\tau\mid h,g})}^2
&=\mathbb E\!\left[\|F_{A,\theta}(Z_\tau,\tau,h,g)-F^*(Z_\tau,\tau,h,g)\|_F^2\mid H_t=h,\mathcal G=g\right] \\
&=\mathbb E\!\left[\|f_\theta(Z_\tau',\tau,h,g)-f(Z_\tau',\tau,h,g)\|_F^2\mid H_t=h,\mathcal G=g\right].
\end{align*}
Therefore, it suffices to approximate the latent-space denoiser $f$ in $L^2(P_{\tau\mid h,g}')$, uniformly over $(\tau,h,g)$.

\medskip
\noindent
\textbf{Step 3: Tail control and choice of the truncation radius.}
Fix $\varepsilon>0$. For $R>0$, define
\[
\mathcal C_R:=[-R,R]^{d\times P}.
\]
By Assumption 3,
\[
\|f(Z',\tau,h,g)\|_F\le C_0+C_1\|Z'\|_F
\qquad \text{for all }(Z',\tau,h,g).
\]
If $Z'\in\mathcal C_R$, then $\|Z'\|_F\le \sqrt{dP}\,R$, and therefore
\[
\|f(Z',\tau,h,g)\|_F\le C_0+C_1\sqrt{dP}\,R=:B_R.
\]
In particular, every coordinate of $f(Z',\tau,h,g)$ belongs to $[-B_R,B_R]$ whenever $Z'\in\mathcal C_R$.

Now let $f_\theta$ be any approximating map satisfying the global bound
\[
\|f_\theta(Z',\tau,h,g)\|_F\le \sqrt{dP}\,B_R
\qquad \text{for all }(Z',\tau,h,g).
\]
Define
\[
I_2(\tau,h,g):=\mathbb E\!\left[
\|f_\theta(Z_\tau',\tau,h,g)-f(Z_\tau',\tau,h,g)\|_F^2
\mathbf 1_{\{Z_\tau'\notin \mathcal C_R\}}
\mid H_t=h,\mathcal G=g
\right].
\]
Using $\|a-b\|_F^2\le 2\|a\|_F^2+2\|b\|_F^2$, we get
\begin{align*}
I_2(\tau,h,g)
&\le 2\,\mathbb E\!\left[\|f_\theta(Z_\tau',\tau,h,g)\|_F^2\mathbf 1_{\{Z_\tau'\notin \mathcal C_R\}}\mid H_t=h,\mathcal G=g\right] \\
&\quad +2\,\mathbb E\!\left[\|f(Z_\tau',\tau,h,g)\|_F^2\mathbf 1_{\{Z_\tau'\notin \mathcal C_R\}}\mid H_t=h,\mathcal G=g\right].
\end{align*}
For the first term, by the global bound on $f_\theta$,
\begin{align*}
\mathbb E\!\left[\|f_\theta(Z_\tau',\tau,h,g)\|_F^2\mathbf 1_{\{Z_\tau'\notin \mathcal C_R\}}\mid H_t=h,\mathcal G=g\right]
&\le dP\,B_R^2\,\mathbb P(Z_\tau'\notin\mathcal C_R\mid H_t=h,\mathcal G=g).
\end{align*}
Hence
\begin{align*}
I_2(\tau,h,g)
&\le 2dP\,B_R^2\,\mathbb P(Z_\tau'\notin\mathcal C_R\mid H_t=h,\mathcal G=g) \\
&\quad +2\,\mathbb E\!\left[\|f(Z_\tau',\tau,h,g)\|_F^2\mathbf 1_{\{Z_\tau'\notin \mathcal C_R\}}\mid H_t=h,\mathcal G=g\right].
\end{align*}
Since $Z_\tau'\notin\mathcal C_R$ implies $\|Z_\tau'\|_F>R$, we further obtain
\begin{align*}
I_2(\tau,h,g)
&\le 2dP\,B_R^2\,\mathbb P(\|Z_\tau'\|_F>R\mid H_t=h,\mathcal G=g) \\
&\quad +2\,\mathbb E\!\left[\|f(Z_\tau',\tau,h,g)\|_F^2\mathbf 1_{\{\|Z_\tau'\|_F>R\}}\mid H_t=h,\mathcal G=g\right].
\end{align*}

By the uniform square-integrability condition in Assumption 3, define
\[
T_R :=
\sup_{\tau\in[\tau_0,1-\tau_0],\,h\in\mathcal H,\,g\in\mathcal G}
\mathbb E\!\left[
\|Z'_\tau\|_F^2
\mathbf 1_{\{\|Z'_\tau\|_F>R\}}
\mid H_t=h,\mathcal G=g
\right].
\]
Then $T_R\to 0$ as $R\to\infty$.

Since $\{\|Z'_\tau\|_F>R\}$ implies $\|Z'_\tau\|_F^2\ge R^2$, we have
\[
\mathbb P(\|Z'_\tau\|_F>R\mid H_t=h,\mathcal G=g)
\le
\frac{1}{R^2}
\mathbb E\!\left[
\|Z'_\tau\|_F^2
\mathbf 1_{\{\|Z'_\tau\|_F>R\}}
\mid H_t=h,\mathcal G=g
\right]
\le \frac{T_R}{R^2}.
\]
Moreover, by Assumption 3,
\[
\|f(Z',\tau,h,g)\|_F^2
\le 2C_0^2+2C_1^2\|Z'\|_F^2.
\]
Therefore,
\begin{align*}
&\mathbb E\!\left[
\|f(Z'_\tau,\tau,h,g)\|_F^2
\mathbf 1_{\{\|Z'_\tau\|_F>R\}}
\mid H_t=h,\mathcal G=g
\right] \\
&\le
2C_0^2
\mathbb P(\|Z'_\tau\|_F>R\mid H_t=h,\mathcal G=g)
+
2C_1^2
\mathbb E\!\left[
\|Z'_\tau\|_F^2
\mathbf 1_{\{\|Z'_\tau\|_F>R\}}
\mid H_t=h,\mathcal G=g
\right] \\
&\le
\left(\frac{2C_0^2}{R^2}+2C_1^2\right)T_R .
\end{align*}
Substituting these estimates into the bound for $I_2(\tau,h,g)$ gives
\[
I_2(\tau,h,g)
\le
2dP\,B_R^2\frac{T_R}{R^2}
+
2\left(\frac{2C_0^2}{R^2}+2C_1^2\right)T_R .
\]
Since $B_R=C_0+C_1\sqrt{dP}\,R$, the ratio $B_R^2/R^2$ is uniformly bounded for $R\ge 1$.
Hence, there exists a constant $K>0$, independent of $(\tau,h,g)$, such that
\[
I_2(\tau,h,g)\le K T_R .
\]
Because $T_R\to 0$, we conclude that
\[
\lim_{R\to\infty}\sup_{\tau,h,g}I_2(\tau,h,g)=0.
\]
Therefore, we may choose $R^*>0$ large enough such that
\[
\sup_{\tau,h,g}I_2(\tau,h,g)\le \varepsilon^2
\]
for every approximating map $f_\theta$ satisfying
\[
\|f_\theta(Z',\tau,h,g)\|_F\le \sqrt{dP}\,B_{R^*}.
\]

\medskip
\noindent
\textbf{Step 4: Approximation on the compact region and conclusion.}
Let
\[
\mathcal S_{R^*}:=\mathcal C_{R^*}\times[\tau_0, 1-\tau_0]\times\mathcal H\times\mathcal G.
\]
By Assumption 2, the function $f$ is jointly continuous on the compact set $\mathcal S_{R^*}$. Hence, by the universal approximation theorem for ReLU networks, there exists a ReLU feedforward neural network
\[
g_\theta:\mathbb R^{d\times P}\times[\tau_0, 1-\tau_0]\times\mathcal H\times\mathcal G\to\mathbb R^{d\times P}
\]
such that
\[
\sup_{(Z',\tau,h,g)\in\mathcal S_{R^*}}
\|g_\theta(Z',\tau,h,g)-f(Z',\tau,h,g)\|_\infty\le \varepsilon.
\]

Now define the scalar clipping map
\[
\Pi_{B_{R^*}}(x):=\max\{-B_{R^*},\min(x,B_{R^*})\},
\]
and define $f_\theta$ coordinatewise by
\[
[f_\theta(Z',\tau,h,g)]_{ij}:=\Pi_{B_{R^*}}([g_\theta(Z',\tau,h,g)]_{ij}).
\]
Since $\Pi_{B_{R^*}}$ can be represented exactly by ReLU operations, $f_\theta$ is still a ReLU feedforward neural network. By construction,
\[
\|f_\theta(Z',\tau,h,g)\|_F\le \sqrt{dP}\,B_{R^*}
\qquad \text{for all }(Z',\tau,h,g),
\]
so the tail estimate derived in Step 3 applies to this $f_\theta$.

Moreover, for every $(Z',\tau,h,g)\in\mathcal S_{R^*}$, each coordinate of $f(Z',\tau,h,g)$ lies in $[-B_{R^*},B_{R^*}]$. Since $\Pi_{B_{R^*}}$ is the metric projection onto this interval, for any $y\in[-B_{R^*},B_{R^*}]$ and any $x\in\mathbb R$,
\[
|\Pi_{B_{R^*}}(x)-y|\le |x-y|.
\]
Applying this coordinatewise with $y=[f(Z',\tau,h,g)]_{ij}$ and $x=[g_\theta(Z',\tau,h,g)]_{ij}$, we obtain
\[
\|f_\theta(Z',\tau,h,g)-f(Z',\tau,h,g)\|_F
\le
\|g_\theta(Z',\tau,h,g)-f(Z',\tau,h,g)\|_F
\]
for all $(Z',\tau,h,g)\in\mathcal S_{R^*}$. Since the output has $dP$ coordinates and the uniform coordinatewise error is at most $\varepsilon$, it follows that
\begin{align*}
\sup_{(Z',\tau,h,g)\in\mathcal S_{R^*}}
\|f_\theta(Z',\tau,h,g)-f(Z',\tau,h,g)\|_F^2
&\le dP\,\varepsilon^2.
\end{align*}

Define
\[
I_1(\tau,h,g):=\mathbb E\!\left[
\|f_\theta(Z_\tau',\tau,h,g)-f(Z_\tau',\tau,h,g)\|_F^2
\mathbf 1_{\{Z_\tau'\in\mathcal C_{R^*}\}}
\mid H_t=h,\mathcal G=g
\right].
\]
Since the integrand is bounded by $dP\,\varepsilon^2$ on $\mathcal C_{R^*}$, we have
\[
I_1(\tau,h,g)\le dP\,\varepsilon^2.
\]
Together with the tail estimate from Step 3,
\[
I_2(\tau,h,g)\le \varepsilon^2.
\]
Hence, by the reduction established in Step 2,
\begin{align*}
\|F_{A,\theta}(\cdot,\tau,h,g)-F^*(\cdot,\tau,h,g)\|_{L^2(P_{\tau\mid h,g})}^2
&=I_1(\tau,h,g)+I_2(\tau,h,g) \\
&\le dP\,\varepsilon^2+\varepsilon^2 \\
&=(dP+1)\varepsilon^2.
\end{align*}
Taking square roots yields
\[
\|F_{A,\theta}(\cdot,\tau,h,g)-F^*(\cdot,\tau,h,g)\|_{L^2(P_{\tau\mid h,g})}
\le \sqrt{dP+1}\,\varepsilon.
\]
Finally, taking the supremum over $\tau\in[\tau_0, 1-\tau_0]$, $h\in\mathcal H$, and $g\in\mathcal G$, we obtain
\[
\sup_{\tau\in[\tau_0, 1-\tau_0],\,h\in\mathcal H,\,g\in\mathcal G}
\|F_{A,\theta}(\cdot,\tau,h,g)-F^*(\cdot,\tau,h,g)\|_{L^2(P_{\tau\mid h,g})}
\le \sqrt{dP+1}\,\varepsilon.
\]
This completes the proof.
\end{proof}

\section{Qualitative FC and SPI Visualizations}

\begin{figure}[H]
\centering
\includegraphics[width=\textwidth]{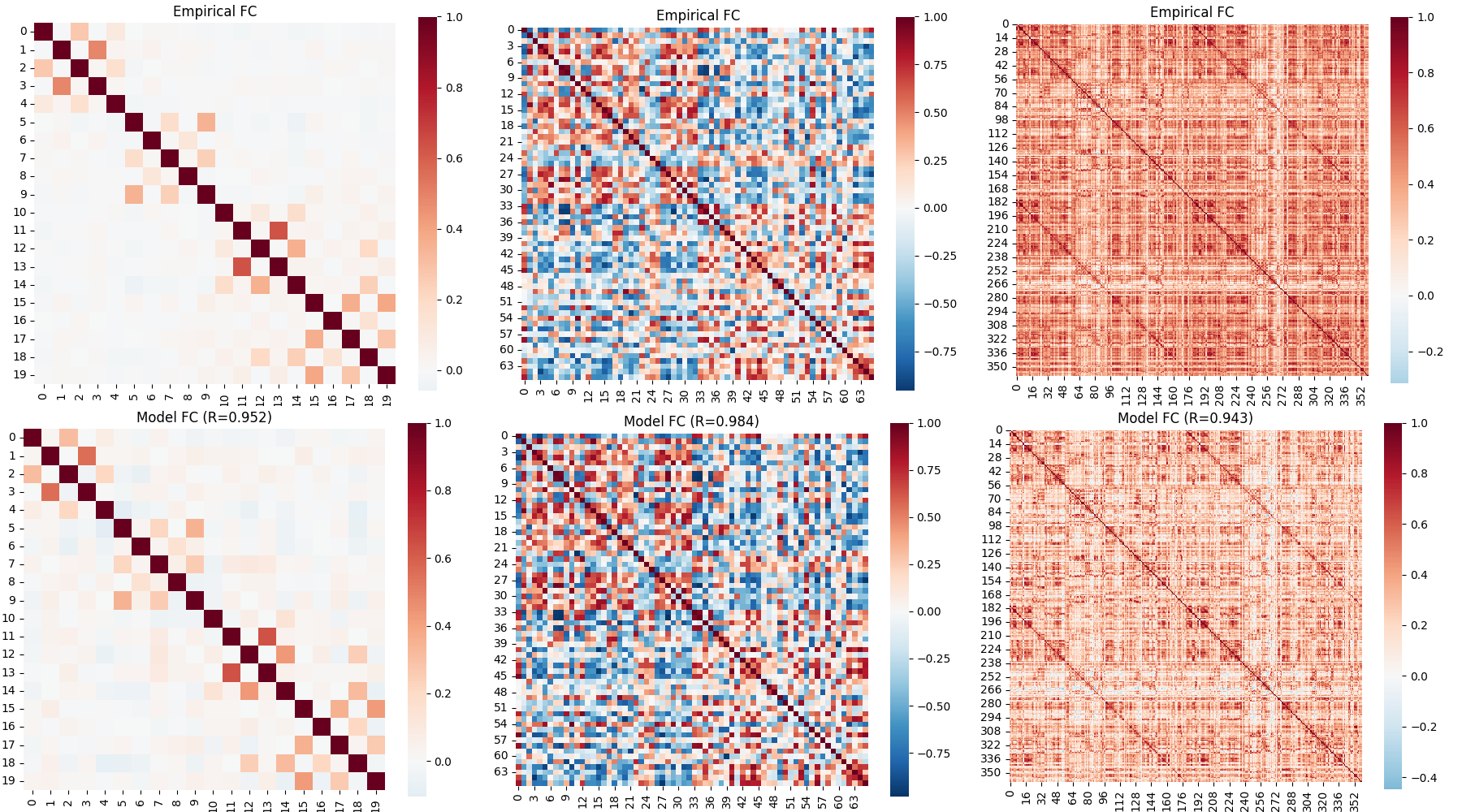}
\caption{Qualitative visualization of functional connectivity recovery on the RNN, SDDEs, and HCP datasets. From left to right, the columns correspond to RNN, SDDEs, and HCP. The top row shows empirical FC matrices, and the bottom row shows FC matrices generated by FAST-Brain.}
\label{fig:fc_visualization}
\end{figure}

\begin{figure}[H]
\centering
\includegraphics[width=\linewidth]{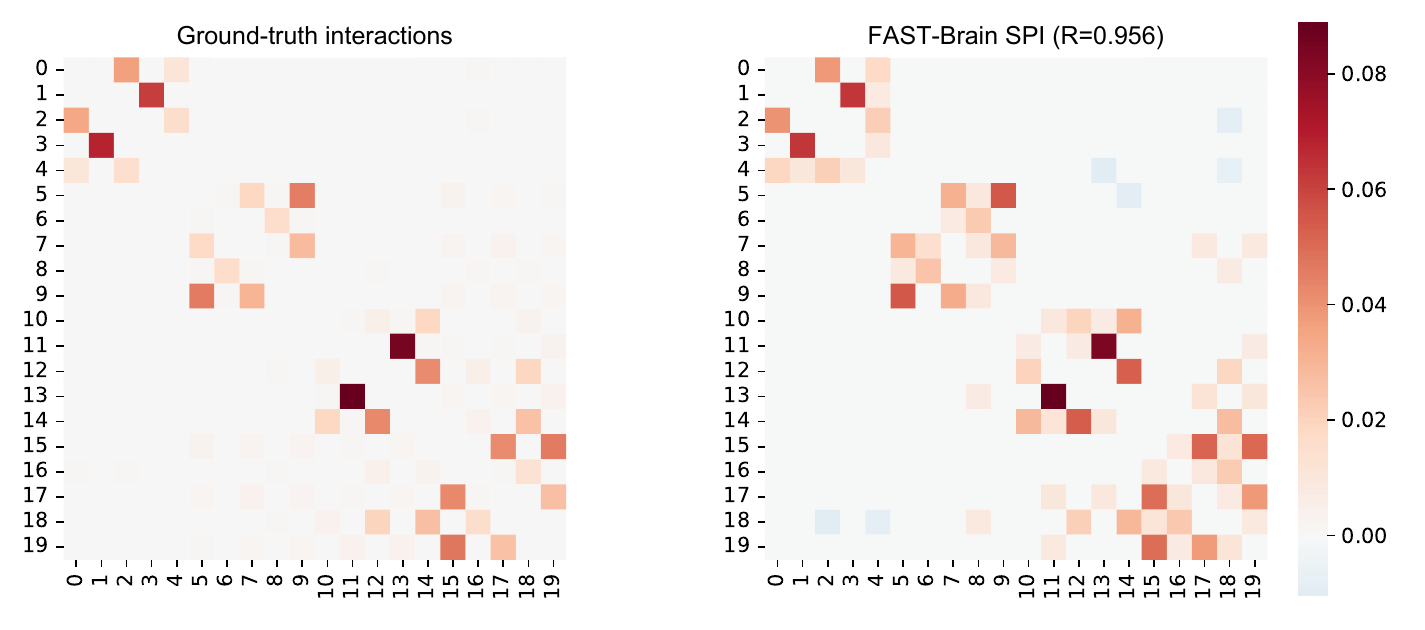}
\caption{Ground-truth interactions (left) and FAST-Brain SPI estimates (right) on the RNN dataset. The panels use separate color scales.}
\label{fig:spi_rnn}
\end{figure}

Figure~\ref{fig:fc_visualization} provides a qualitative comparison between empirical FC matrices and FC matrices recovered from FAST-Brain generated sequences. From left to right, the three columns correspond to the RNN, SDDEs, and HCP datasets, with empirical FC matrices shown in the top row and generated FC matrices shown in the bottom row. FAST-Brain recovers the localized and network-structured connectivity patterns in the RNN dataset, suggesting that the model can learn the underlying coupling information induced by the synthetic dynamical system. For the more complex SDDEs and high-dimensional HCP datasets, FAST-Brain still preserves the major FC organization, including block-like structures, off-diagonal dependencies, and global correlation patterns. These visual results are consistent with the quantitative FC results in Table~\ref{tab:fc_main}, further supporting that FAST-Brain generates surrogate BOLD sequences with faithful functional connectivity structure even in high-dimensional settings.
\section{Dataset Generation}
\subsection{RNN Dataset}

We simulate a brain network consisting of $N=20$ nodes partitioned into $\mathcal{K}=4$ functional communities, where $c_i \in \{1,...,K\}$ denotes the network assignment of node $i$. The nodes are spatially embedded with coordinates $x_i \in \mathbb{R}^3$, and the Fiber Length matrix $L$ is derived from Euclidean distances adjusted by a cross-hemispheric penalty $\delta$, defined as $$L_{ij} = ||x_i - x_j||_2 + \delta \cdot \mathbb{I}(hemi_i \ne hemi_j),$$ Using this, the Structural Connectivity matrix $SC$ (fiber count) is generated via a distance-dependent stochastic block model. The expected connection intensity is determined by $$\lambda_{ij} = S \cdot B_{c_i c_j} \cdot \exp(-\alpha L_{ij}),$$ where $S$ is a scaling factor, $B$ is the network affinity matrix, and $\alpha$ is an exponential spatial decay, yielding $SC_{ij} \sim Poisson(\lambda_{ij})$. The brain activity is then modeled using an RNN with a state vector $x(t) \in \mathbb{R}^N$. The neural dynamics follow the stochastic differential equation $$dx(t) = [-x(t) + SC\, h(x(t))]dt + \sigma d\xi(t)$$, where $h(\cdot) = \tanh(\cdot)$ and $\xi(t)$ is an $N$-dimensional Gaussian noise process. To simulate these continuous dynamics, we employ Euler discretization to obtain the update rule $x(t+\Delta t) = x(t) + [-x(t) + SC\, h(x(t))]\Delta t + \sigma \sqrt{\Delta t} Z(t)$, with $Z(t) \sim \mathcal{N}(0, I_N)$. 

\subsection{SDDEs Dataset}
The signal transmission delay $\Gamma_{ij}$ and the effective coupling weight $W_{ij}$ (scaled by global coupling $G$) dictate the network input:
\begin{align*}
    \Gamma_{ij} &= \frac{K_{ij}}{c} \\
    W_{ij} &= G \cdot (A_{ij} \odot W_{SC,ij}) \\
    I_i(t) &= \sum_{j=1}^{N} W_{ij} V_j(t - \Gamma_{ij}) + \xi_i(t)
\end{align*}
where $c$ is the conduction velocity, $\odot$ is the Hadamard product, and $\xi_i(t) \sim \mathcal{N}(0, \sigma^2)$ is the baseline Gaussian white noise.

The local neural activity (latent state) for each node $i$ is driven by a Generic 2D Oscillator, yielding a system of Stochastic Delay Differential Equations (SDDEs):
\begin{align*}
    dV_i(t) &= \eta \cdot \tau \left[ -f V_i(t)^3 + e V_i(t)^2 + g V_i(t) + \alpha W_i(t) + \gamma I_i(t) \right] dt \\
    dW_i(t) &= \frac{\eta}{\tau_w} \left[ c V_i(t)^2 + b V_i(t) - \beta W_i(t) + \alpha_2 \right] dt
\end{align*}
where $V_i(t)$ represents the excitatory Local Field Potential (LFP), $W_i(t)$ is the slow inhibitory recovery variable, and $\{\eta, \tau, f, e, g, \alpha, \gamma, \tau_w, c, b, \beta, \alpha_2\}$ are fixed non-linear system parameters.

The continuous macroscopic observation $y_i(t)$ is obtained via convolution of the high-frequency latent state $V_i(t)$ with the first-order Volterra Hemodynamic Response Function (HRF) kernel $h(t)$:
\begin{align*}
    h(t) &= \frac{1}{3} \exp\left(-\frac{0.5t}{0.8}\right) \frac{\sin(\omega t)}{\omega}, \quad \text{with } \omega = \sqrt{\frac{1}{0.4} - \frac{1}{4 \times 0.8^2}} \\
    y_i(t) &= (V_i * h)(t) = \int_{0}^{t} V_i(s) h(t - s) ds
\end{align*}
The discrete fMRI BOLD measurement matrix $\mathbf{B} \in \mathbb{R}^{M \times N}$ is obtained by down-sampling $y_i(t)$ at a given Repetition Time ($TR$):
\begin{equation*}
    B_{m,i} = y_i(m \cdot TR), \quad m \in \{1, 2, \dots, M\}
\end{equation*}

\subsection{HCP dataset}
For the real rs-fMRI experiments, we use the same HCP data setting as NPI \citep{Luo2025}, following their HCP-MMP 360-ROI representation and preprocessing protocol.

For HCP structural priors, diffusion MRI tractography streamlines are
assigned to ROI pairs in the Glasser/HCP-MMP 360-ROI atlas, consistent
with atlas-based connectome construction~\citep{li2026connectome,zhang2018mapping}.
$SC$ encodes streamline-based connection strength, and $L$ the mean
physical streamline length between each pair. Both matrices are
symmetrized, normalized to $[0,1]$, and used as subject-level priors in
the GCN branch.

\section{Implementation Details}

All FAST-Brain experiments and ablations use the same implementation, with only dataset-specific configurations changed. We split each time series chronologically: the first 80\% of the temporal samples are used for training, and the remaining temporal segment is reserved for rolling evaluation. Sliding windows are constructed within the corresponding split. The model conditions on a history window and predicts a future block of length $P=24$. The history length is set to $S=60$ for RNN and SDDEs, and to $S=10$ for HCP. For RNN and SDDEs, each node signal is independently min-max normalized to $[-1,1]$ before window construction. For HCP, the first 30 frames of each REST session are discarded, the first 360 Glasser ROIs are retained, and each session is standardized independently before window construction. In the structural branch, we use normalized graph powers up to order $K=2$. During sampling, the flow ODE is solved using explicit Euler discretization with 20 steps. 

All the experiments are completed on a single NVIDIA H100 GPU.

\section{Evaluation}

\subsection{Definitions of Evaluation Metrics}
\label{app:metric_definitions}

For completeness, we provide the definitions of the first eight evaluation metrics reported in our evaluation output. These metrics are grouped into two parts: unconditional generation metrics and conditional generation metrics. Specifically, the unconditional generation metrics are \emph{Discriminative Score}, \emph{Predictive Score}, \emph{Context-FID}, and \emph{Correlational Score}; the conditional generation metrics are \emph{CRPS}, \emph{QICE}, \emph{ProbCorr}, and \emph{Conditional FID}.

Let $\{X^{(i)}\}_{i=1}^{n}$ denote real samples and $\{\widetilde{X}^{(i)}\}_{i=1}^{n}$ denote generated samples. In conditional evaluation, let $H^{(b)}$ be the historical context for block $b$, let $Y^{(b)} \in \mathbb{R}^{N \times P}$ be the corresponding true future sequence, and let
\[
\widehat{Y}^{(b,1)},\ldots,\widehat{Y}^{(b,S)} \in \mathbb{R}^{N \times P}
\]
be $S$ stochastic predictions generated under the same condition $H^{(b)}$, where $P$ is the prediction horizon and $N$ is the number of ROIs.

\noindent\textbf{Conditional generation.}

\noindent\textbf{Continuous Ranked Probability Score (CRPS) \citep{matheson1976scoring}.} CRPS evaluates the overall probabilistic accuracy of conditional forecasts. For a scalar target $y$ and predictive cumulative distribution function $F$, CRPS is defined as
\[
\mathrm{CRPS}(F,y)
=
\int_{-\infty}^{\infty}
\bigl(F(z)-\mathbbm{1}\{z\ge y\}\bigr)^2\,dz.
\]
In our implementation, the predictive distribution is represented empirically by $S$ stochastic samples. For empirical samples $\widehat y_1,\ldots,\widehat y_S$, CRPS admits the equivalent form
\[
\mathrm{CRPS}
=
\frac{1}{S}\sum_{s=1}^{S}|\widehat y_s-y|
-
\frac{1}{2S^2}\sum_{s=1}^{S}\sum_{s'=1}^{S}|\widehat y_s-\widehat y_{s'}|.
\]
We average this quantity over all forecasted entries across blocks, horizons, and ROIs. Lower values indicate better calibrated and sharper probabilistic forecasts.

\noindent\textbf{Quantile Interval Coverage Error (QICE) \citep{han2022card}.} QICE measures the calibration of predictive intervals. For each target entry, we form $K$ quantile bins from the empirical predictive distribution and record into which bin the true observation falls. Let $r_k$ be the empirical proportion of true observations falling into bin $k$, and let the ideal proportion be $1/K$. Then
\[
\mathrm{QICE}
=
\frac{1}{K}\sum_{k=1}^{K}\left|r_k-\frac{1}{K}\right|.
\]
A smaller QICE means that the empirical coverage of predictive quantiles is closer to the nominal uniform target, hence the predictive distribution is better calibrated.

\noindent\textbf{Probabilistic Correlation Score (ProbCorr) \citep{ni2021sig}.} ProbCorr measures how well the model preserves correlation structure in conditional stochastic forecasts. For each block $b$, let $C^{(b)}_{\mathrm{true}}$ be the correlation matrix of the true future sequence $Y^{(b)}$, and let $C^{(b,s)}_{\mathrm{pred}}$ be the correlation matrix of the $s$-th generated future sample $\widehat{Y}^{(b,s)}$. We define
\[
\mathrm{ProbCorr}
=
\frac{1}{B}\sum_{b=1}^{B}
\frac{1}{S}\sum_{s=1}^{S}
\frac{1}{N^2}
\left\|
C^{(b,s)}_{\mathrm{pred}}-C^{(b)}_{\mathrm{true}}
\right\|_1.
\]
Thus, ProbCorr is the average entrywise absolute difference between predicted and true correlation matrices, further averaged over all stochastic samples and all conditional blocks. Lower values indicate better preservation of conditional correlation structure.

\noindent\textbf{Conditional FID \citep{yue2022ts2vec}.} Conditional FID evaluates alignment between real and generated conditional sequence distributions. For each context block $H^{(b)}$, we concatenate the context with the true future $Y^{(b)}$ to form a real conditional sequence, and concatenate the same context with each sampled prediction $\widehat{Y}^{(b,s)}$ to form generated conditional sequences. These sequences are embedded by a representation encoder $\phi(\cdot)$, producing real features with mean and covariance $(\mu_c^{r},\Sigma_c^{r})$ and generated features with mean and covariance $(\mu_c^{g},\Sigma_c^{g})$. The metric is
\[
\mathrm{Conditional\ FID}
=
\|\mu_c^{r}-\mu_c^{g}\|_2^2
+
\mathrm{Tr}\!\left(
\Sigma_c^{r}+\Sigma_c^{g}
-2(\Sigma_c^{r}\Sigma_c^{g})^{1/2}
\right).
\]
A smaller value indicates that the generated futures, when paired with the same history, better match the real conditional distribution in representation space.

\noindent\textbf{Unconditional generation.}

\noindent\textbf{Discriminative Score. \citep{yoon2019time}} This metric measures how easily generated samples can be distinguished from real samples. We train a binary classifier $D(\cdot)$ on real and generated sequences, with labels $1$ for real samples and $0$ for generated samples. Let $\mathrm{Acc}$ denote the classification accuracy on a held-out set. The discriminative score is defined as
\[
\mathrm{DS} = \left| \mathrm{Acc} - 0.5 \right|.
\]
If generated samples are indistinguishable from real samples, then the classifier cannot do better than random guessing and $\mathrm{DS}$ approaches $0$. A smaller value therefore indicates better generation quality.

\noindent\textbf{Predictive Score. \citep{yoon2019time}} This metric evaluates whether generated data preserve temporal predictive structure. A predictor $f_{\mathrm{pred}}$ is trained on generated sequences to forecast the next-step or target future values, and is then evaluated on real sequences. If $x^{(i)}_{T+1}$ denotes the true target and $\widehat{x}^{(i)}_{T+1}=f_{\mathrm{pred}}(x^{(i)}_{1:T})$ denotes the prediction, then the predictive score is computed by mean absolute error:
\[
\mathrm{PS} = \frac{1}{n}\sum_{i=1}^{n}\left|x^{(i)}_{T+1}-\widehat{x}^{(i)}_{T+1}\right|.
\]
Lower values indicate that the generated data better preserve the forecasting-relevant temporal structure.

\noindent\textbf{Context-FID. \citep{yuan2024diffusion}} This metric compares real and generated sequences in a learned contextual feature space. Let $\phi(\cdot)$ be the embedding extracted from a sequence encoder, and let
$
\mu_r,\Sigma_r \quad \text{and} \quad \mu_g,\Sigma_g
$
be the empirical means and covariance matrices of the real and generated embeddings, respectively. The Fr\'echet distance between the two Gaussian approximations is
\[
\mathrm{Context\mbox{-}FID}
=
\|\mu_r-\mu_g\|_2^2
+
\mathrm{Tr}\!\left(
\Sigma_r+\Sigma_g-2(\Sigma_r\Sigma_g)^{1/2}
\right).
\]
A smaller value means that the generated sequences are closer to the real ones in contextual representation space.

\noindent\textbf{Correlational Score. \citep{yuan2024diffusion}} This metric measures whether the pairwise dependency structure is preserved. For each sample, we compute its ROI-wise correlation matrix, and then compare the average correlation matrix of generated samples with that of real samples. Let $C_r$ and $C_g$ denote the corresponding average correlation matrices. The correlational score is defined as
\[
\mathrm{CS}
=
\frac{1}{N^2}\|C_r-C_g\|_1,
\]
that is, the mean absolute entrywise deviation between the two correlation matrices. Lower values indicate better preservation of dependence structure.

\subsection{Other Results}

\noindent\emph{HCP FC recovery.}
Across 10 generation runs per method with fixed trained models, FAST-Brain achieves higher HCP FC correlation than FlowTS (one-sided, unpaired Welch's $t$-test, $p<10^{-10}$). The 95\% percentile bootstrap confidence interval for the mean difference (FAST-Brain minus FlowTS) is $[0.203,0.214]$, using 10,000 bootstrap replicates with independent resampling of 10 values with replacement within each method (NumPy \texttt{default\_rng(42)}). Long-horizon rollouts on 360-ROI signals may
amplify local prediction errors into FC distortion. We hypothesize
that clean-BOLD prediction and spatial priors help preserve
interregional correlations, consistent with the ablations in
Table~\ref{tab:ablation_main}.

\textbf{Conditional Generation.}
Tables~\ref{tab:cond_main2} and~\ref{tab:cond_main3} report additional conditional generation results on SDDEs and HCP. On HCP, FAST-Brain achieves the best performance on CRPS, ProbCorr, and Conditional FID. This indicates that FAST-Brain improves probabilistic accuracy, correlation preservation, and feature-level distributional alignment in real high-dimensional rs-fMRI data. On SDDEs, FAST-Brain also achieves the best CRPS and QICE, showing strong marginal probabilistic accuracy and calibration.

For the remaining SDDEs metrics, FAST-Brain is competitive but not the best: TimeDiff obtains the lowest ProbCorr, while FlowTS achieves the lowest Conditional FID. This may reflect the stochastic delay effects, strong temporal coupling, and periodic structure in SDDEs. Some baselines may better fit local periodic patterns or feature-level statistics, leading to advantages on specific conditional metrics. In contrast, FAST-Brain prioritizes direct clean-signal recovery with spatial priors, which leads to stronger global FC recovery and latent-geometry preservation in the main experiments, even when individual conditional metrics are not uniformly optimal.
\begin{table}[H]
\centering
\renewcommand{\arraystretch}{1.2}
\scriptsize
\caption{Conditional generation metrics in SDDEs.}
\label{tab:cond_main2}
\ExpTableFont
\setlength{\tabcolsep}{2pt}

\begin{tabular}{@{}lcccccccc@{}}
\toprule
Metric & NPI & TimeDiff & SSSD & DiffusionTS & FlowTS & DSFM & LS4 & FAST-Brain \\
\midrule
CRPS$\downarrow$            & 0.213$_{\scriptstyle \pm 0.002}$ & 0.253$_{\scriptstyle \pm 0.000}$ & 0.340$_{\scriptstyle \pm 0.000}$ & 0.263$_{\scriptstyle \pm 0.001}$ & 0.411$_{\scriptstyle \pm 0.001}$ & 0.240$_{\scriptstyle \pm 0.000}$ & 0.266$_{\scriptstyle \pm 0.000}$ & 0.198$_{\scriptstyle \pm 0.003}$ \\
QICE$\downarrow$            & 0.160$_{\scriptstyle \pm 0.000}$ & 0.125$_{\scriptstyle \pm 0.000}$ & 0.035$_{\scriptstyle \pm 0.000}$ & 0.023$_{\scriptstyle \pm 0.000}$ & 0.105$_{\scriptstyle \pm 0.000}$ & 0.152$_{\scriptstyle \pm 0.000}$ & 0.126$_{\scriptstyle \pm 0.000}$ & 0.019$_{\scriptstyle \pm 0.002}$ \\
ProbCorr$\downarrow$        & 0.310$_{\scriptstyle \pm 0.006}$ & 0.215$_{\scriptstyle \pm 0.000}$ & 0.422$_{\scriptstyle \pm 0.000}$ & 0.238$_{\scriptstyle \pm 0.001}$ & 0.539$_{\scriptstyle \pm 0.001}$ & 0.247$_{\scriptstyle \pm 0.000}$ & 0.284$_{\scriptstyle \pm 0.000}$ & 0.248$_{\scriptstyle \pm 0.003}$ \\
Conditional FID$\downarrow$ & 6.788$_{\scriptstyle \pm 0.348}$ & 4.302$_{\scriptstyle \pm 0.043}$ & 19.256$_{\scriptstyle \pm 0.066}$ & 3.741$_{\scriptstyle \pm 0.010}$ & 3.520$_{\scriptstyle \pm 0.005}$ & 9.174$_{\scriptstyle \pm 1.367}$ & 6.015$_{\scriptstyle \pm 0.002}$ & 5.841$_{\scriptstyle \pm 0.049}$ \\
\bottomrule
\end{tabular}
\end{table}

\begin{table}[H]
\centering
\caption{Conditional generation metrics in HCP.}
\label{tab:cond_main3}
\ExpTableFont
\renewcommand{\arraystretch}{1.2}
\scriptsize
\setlength{\tabcolsep}{2pt}

\resizebox{\linewidth}{!}{\begin{tabular}{@{}lcccccccc@{}}
\toprule
Metric & NPI & TimeDiff & SSSD & DiffusionTS & FlowTS & DSFM & LS4 & FAST-Brain \\
\midrule
CRPS$\downarrow$            & 0.710$_{\scriptstyle \pm 0.017}$ & 0.829$_{\scriptstyle \pm 0.000}$ & 0.580$_{\scriptstyle \pm 0.000}$ & 0.640$_{\scriptstyle \pm 0.000}$ & 0.634$_{\scriptstyle \pm 0.002}$ & 0.566$_{\scriptstyle \pm 0.002}$ & 0.680$_{\scriptstyle \pm 0.001}$ & 0.559$_{\scriptstyle \pm 0.002}$ \\
QICE$\downarrow$            & 0.160$_{\scriptstyle \pm 0.000}$ & 0.155$_{\scriptstyle \pm 0.000}$ & 0.008$_{\scriptstyle \pm 0.001}$ & 0.123$_{\scriptstyle \pm 0.000}$ & 0.095$_{\scriptstyle \pm 0.001}$ & 0.028$_{\scriptstyle \pm 0.000}$ & 0.153$_{\scriptstyle \pm 0.000}$ & 0.027$_{\scriptstyle \pm 0.000}$ \\
ProbCorr$\downarrow$        & 0.464$_{\scriptstyle \pm 0.006}$ & 0.485$_{\scriptstyle \pm 0.000}$ & 0.566$_{\scriptstyle \pm 0.000}$ & 0.494$_{\scriptstyle \pm 0.000}$ & 0.551$_{\scriptstyle \pm 0.001}$ & 0.553$_{\scriptstyle \pm 0.001}$ & 0.532$_{\scriptstyle \pm 0.000}$ & 0.450$_{\scriptstyle \pm 0.001}$ \\
Conditional FID$\downarrow$ & 81.185$_{\scriptstyle \pm 1.924}$ & 72.006$_{\scriptstyle \pm 1.618}$ & 21.527$_{\scriptstyle \pm 0.193}$ & 15.274$_{\scriptstyle \pm 0.013}$ & 15.909$_{\scriptstyle \pm 0.102}$ & 31.333$_{\scriptstyle \pm 1.290}$ & 14.341$_{\scriptstyle \pm 0.004}$ & 13.882$_{\scriptstyle \pm 0.071}$ \\
\bottomrule
\end{tabular}}
\end{table}
\textbf{Unconditional Generation.}
Although FAST-Brain is trained as a conditional future generator, unconditional generation metrics are computed by rolling the model forward autoregressively to produce full-length sequences without using future observations. Tables~\ref{tab:uncond_main1}--\ref{tab:uncond_main3} summarize unconditional generation results across RNN, SDDEs, and HCP. FAST-Brain shows strong overall performance, achieving the best Discriminative Score and Correlational Score on all three datasets. This indicates that the generated sequences are consistently more realistic and better preserve the temporal correlation structure of real brain signals. FAST-Brain also obtains the best Context-FID on RNN and SDDEs, suggesting improved alignment between generated and real feature distributions in both synthetic settings.

On HCP, FAST-Brain remains highly competitive on Context-FID and substantially outperforms most baselines in Discriminative Score and Correlational Score. Although FlowTS achieves slightly better Predictive Score and Context-FID on HCP, FAST-Brain provides stronger overall sample realism and correlation preservation. These results suggest that direct clean-signal prediction with spatial priors helps constrain generated trajectories toward structurally faithful brain dynamics, especially in terms of global temporal dependence and connectivity-related statistics.
\begin{table}[H]
\centering
\caption{Unconditional generation metrics in RNN.}
\label{tab:uncond_main1}
\ExpTableFont
\renewcommand{\arraystretch}{1.2}
\scriptsize
\setlength{\tabcolsep}{2pt}

\begin{tabular}{@{}lcccccccc@{}}
\toprule
Metric & NPI & TimeDiff & SSSD & DiffusionTS & FlowTS & DSFM & LS4 & FAST-Brain \\
\midrule
Discriminative Score$\downarrow$ & 0.330$_{\scriptstyle \pm 0.138}$ & 0.360$_{\scriptstyle \pm 0.145}$ & 0.165$_{\scriptstyle \pm 0.092}$ & 0.200$_{\scriptstyle \pm 0.097}$ & 0.115$_{\scriptstyle \pm 0.063}$ & 0.100$_{\scriptstyle \pm 0.081}$ & 0.125$_{\scriptstyle \pm 0.145}$ & 0.085$_{\scriptstyle \pm 0.059}$ \\
Predictive Score$\downarrow$     & 0.230$_{\scriptstyle \pm 0.006}$ & 0.211$_{\scriptstyle \pm 0.002}$ & 0.211$_{\scriptstyle \pm 0.002}$ & 0.213$_{\scriptstyle \pm 0.002}$ & 0.208$_{\scriptstyle \pm 0.001}$ & 0.211$_{\scriptstyle \pm 0.002}$ & 0.211$_{\scriptstyle \pm 0.002}$ & 0.209$_{\scriptstyle \pm 0.003}$ \\
Context-FID$\downarrow$          & 1.722$_{\scriptstyle \pm 0.666}$ & 1.508$_{\scriptstyle \pm 0.493}$ & 0.234$_{\scriptstyle \pm 0.052}$ & 0.253$_{\scriptstyle \pm 0.072}$ & 0.154$_{\scriptstyle \pm 0.021}$ & 0.619$_{\scriptstyle \pm 0.060}$ & 0.204$_{\scriptstyle \pm 0.044}$ & 0.121$_{\scriptstyle \pm 0.029}$ \\
Correlational Score$\downarrow$  & 0.380$_{\scriptstyle \pm 0.102}$ & 0.217$_{\scriptstyle \pm 0.003}$ & 0.054$_{\scriptstyle \pm 0.001}$ & 0.058$_{\scriptstyle \pm 0.003}$ & 0.055$_{\scriptstyle \pm 0.001}$ & 0.067$_{\scriptstyle \pm 0.002}$ & 0.058$_{\scriptstyle \pm 0.003}$ & 0.038$_{\scriptstyle \pm 0.001}$ \\
\bottomrule
\end{tabular}
\end{table}

\begin{table}[H]
\centering
\caption{Unconditional generation metrics in SDDEs.}
\label{tab:uncond_main2}
\ExpTableFont
\renewcommand{\arraystretch}{1.2}
\scriptsize
\setlength{\tabcolsep}{2pt}

\resizebox{\linewidth}{!}{\begin{tabular}{@{}lcccccccc@{}}
\toprule
Metric & NPI & TimeDiff & SSSD & DiffusionTS & FlowTS & DSFM & LS4 & FAST-Brain \\
\midrule
Discriminative Score$\downarrow$ & 0.270$_{\scriptstyle \pm 0.149}$ & 0.250$_{\scriptstyle \pm 0.175}$ & 0.350$_{\scriptstyle \pm 0.143}$ & 0.310$_{\scriptstyle \pm 0.158}$ & 0.340$_{\scriptstyle \pm 0.174}$ & 0.370$_{\scriptstyle \pm 0.078}$ & 0.320$_{\scriptstyle \pm 0.098}$ & 0.120$_{\scriptstyle \pm 0.108}$ \\
Predictive Score$\downarrow$     & 0.248$_{\scriptstyle \pm 0.011}$ & 0.283$_{\scriptstyle \pm 0.006}$ & 0.510$_{\scriptstyle \pm 0.023}$ & 0.399$_{\scriptstyle \pm 0.030}$ & 0.550$_{\scriptstyle \pm 0.011}$ & 0.442$_{\scriptstyle \pm 0.020}$ & 0.436$_{\scriptstyle \pm 0.008}$ & 0.334$_{\scriptstyle \pm 0.020}$ \\
Context-FID$\downarrow$          & 10.344$_{\scriptstyle \pm 0.303}$ & 13.530$_{\scriptstyle \pm 0.786}$ & 23.649$_{\scriptstyle \pm 3.418}$ & 10.998$_{\scriptstyle \pm 2.183}$ & 15.747$_{\scriptstyle \pm 0.696}$ & 42.045$_{\scriptstyle \pm 3.913}$ & 11.391$_{\scriptstyle \pm 0.131}$ & 6.532$_{\scriptstyle \pm 0.806}$ \\
Correlational Score$\downarrow$  & 0.531$_{\scriptstyle \pm 0.006}$ & 0.418$_{\scriptstyle \pm 0.001}$ & 0.379$_{\scriptstyle \pm 0.002}$ & 0.360$_{\scriptstyle \pm 0.038}$ & 0.480$_{\scriptstyle \pm 0.002}$ & 0.398$_{\scriptstyle \pm 0.009}$ & 0.385$_{\scriptstyle \pm 0.000}$ & 0.329$_{\scriptstyle \pm 0.015}$ \\
\bottomrule
\end{tabular}}
\end{table}

\begin{table}[H]
\centering
\caption{Unconditional generation metrics in HCP.}
\label{tab:uncond_main3}
\ExpTableFont
\renewcommand{\arraystretch}{1.2}
\scriptsize
\setlength{\tabcolsep}{2pt}

\resizebox{\linewidth}{!}{\begin{tabular}{@{}lcccccccc@{}}
\toprule
Metric & NPI & TimeDiff & SSSD & DiffusionTS & FlowTS & DSFM & LS4 & FAST-Brain \\
\midrule
Discriminative Score$\downarrow$ & 0.287$_{\scriptstyle \pm 0.115}$ & 0.338$_{\scriptstyle \pm 0.113}$ & 0.262$_{\scriptstyle \pm 0.098}$ & 0.275$_{\scriptstyle \pm 0.109}$ & 0.300$_{\scriptstyle \pm 0.083}$ & 0.263$_{\scriptstyle \pm 0.118}$ & 0.275$_{\scriptstyle \pm 0.061}$ & 0.200$_{\scriptstyle \pm 0.115}$ \\
Predictive Score$\downarrow$     & 0.746$_{\scriptstyle \pm 0.019}$ & 0.776$_{\scriptstyle \pm 0.008}$ & 0.795$_{\scriptstyle \pm 0.010}$ & 0.796$_{\scriptstyle \pm 0.010}$ & 0.692$_{\scriptstyle \pm 0.006}$ & 0.702$_{\scriptstyle \pm 0.012}$ & 0.621$_{\scriptstyle \pm 0.008}$ & 0.740$_{\scriptstyle \pm 0.034}$ \\
Context-FID$\downarrow$          & 66.432$_{\scriptstyle \pm 9.314}$ & 93.019$_{\scriptstyle \pm 6.584}$ & 70.196$_{\scriptstyle \pm 2.164}$ & 85.045$_{\scriptstyle \pm 5.285}$ & 58.411$_{\scriptstyle \pm 4.880}$ & 82.183$_{\scriptstyle \pm 9.055}$ & 59.161$_{\scriptstyle \pm 3.337}$ & 59.574$_{\scriptstyle \pm 3.787}$ \\
Correlational Score$\downarrow$  & 0.258$_{\scriptstyle \pm 0.053}$ & 0.361$_{\scriptstyle \pm 0.001}$ & 0.339$_{\scriptstyle \pm 0.001}$ & 0.242$_{\scriptstyle \pm 0.020}$ & 0.282$_{\scriptstyle \pm 0.002}$ & 0.210$_{\scriptstyle \pm 0.015}$ & 0.229$_{\scriptstyle \pm 0.011}$ & 0.162$_{\scriptstyle \pm 0.004}$ \\
\bottomrule
\end{tabular}}
\end{table}

\section{Gaussian Smoothing and Approximate Subspaces}
\label{app:w1extensions}

\begin{lemma}[Automatic regularity of the noisy latent density]
\label{lem:w1smoothing}
Fix $(h,g)$ and $\tau\in[\tau_0,1-\tau_0]$, where $\tau_0\in(0,1/2)$.
Let $\mu_{h,g}$ be the conditional law of $S_t$ and suppose
$Z_\tau'=\tau S_t+(1-\tau)\epsilon'$, with
$\epsilon'\sim\mathcal{N}(0,I_{dP})$ independent of the signal and
conditioning variables. Upon vectorizing $Z_\tau'$ and $S_t$, the density is
\[
p'_{\tau\mid h,g}(z')
=\int\varphi_{(1-\tau)^2 I_{dP}}(z'-\tau s)\,\mu_{h,g}(ds),
\]
where $\varphi_\Sigma$ is the centered Gaussian density with covariance
$\Sigma$. This density is strictly positive and $C^\infty$ in $z'$, even
when $\mu_{h,g}$ has no density.
\end{lemma}

\begin{proof}
The Gaussian kernel is strictly positive everywhere, so its integral
against a probability measure is strictly positive. Each derivative of
the kernel in $z'$ is a polynomial times a Gaussian and is bounded over
its argument. Since $1-\tau\ge\tau_0>0$, these derivatives have constant
integrable envelopes with respect to $\mu_{h,g}$. Differentiating under the
integral and applying dominated convergence establishes smoothness of
every order. This argument establishes regularity in the noisy observation;
joint continuity in $(\tau,h,g)$ remains part of
Assumption~\ref{ass:regularity}.
\end{proof}

\paragraph{Relation to the latent conditional mean.}
Under the moment conditions of Assumption~\ref{ass:growth}, the latent
Tweedie identity gives
\[
\begin{aligned}
f(z',\tau,h,g)
&=\frac{1}{\tau}\bigl(z'+(1-\tau)^2\nabla\log p'_{\tau\mid h,g}(z')\bigr)\\
&=\mathbb{E}[S_t\mid Z_\tau'=z',H_t=h,\mathcal{G}=g].
\end{aligned}
\]
Thus the score expression in Assumption~\ref{ass:regularity} is the latent
conditional mean used in the approximate-subspace result below.
Gaussian smoothing removes the separate density-smoothness requirement;
the joint-continuity, growth, and tail conditions are retained.

\paragraph{Role of the projection basis.}
The basis $A$ describes the population signal subspace and is not required
as an additional input to FAST-Brain. For a fixed population subspace, the
projection and reconstruction maps can be absorbed into the affine input
and output layers of an ambient-space feedforward network. Theorem~\ref{thm:lowddenoiser}
therefore provides a representation and approximation result; it does not
establish that training recovers the population subspace or provide a
finite-sample guarantee for learning it.

\begin{corollary}[Controlled off-subspace residual]
\label{cor:w1approxsubspace}
Fix $\tau_0\in(0,1/2)$ and replace Assumption~\ref{ass:subspace} by
$Y_t=AS_t+R_t$, with $A^\top A=I_d$, $A^\top R_t=0$ almost surely,
$S_t$ and $R_t$ conditionally independent given $(H_t,\mathcal{G})$, and
\[
\sup_{h\in\mathcal{H},\,g\in\mathcal{G}}
\mathbb{E}\!\left[\|R_t\|_F^2\mid H_t=h,\mathcal{G}=g\right]\le\delta^2.
\]
Let $Z_\tau=\tau Y_t+(1-\tau)\epsilon$, where $\epsilon$ has independent
standard Gaussian entries and is independent of $(S_t,R_t,H_t,\mathcal{G})$.
Retain Assumptions~\ref{ass:regularity}--\ref{ass:growth} for
$f(z',\tau,h,g)=\mathbb{E}[S_t\mid A^\top Z_\tau=z',H_t=h,\mathcal{G}=g]$
and the projected noisy state $Z_\tau'=A^\top Z_\tau$.
Writing $\Pi_\perp=I_N-AA^\top$, the Bayes-optimal denoiser satisfies
\[
F^*(Z,\tau,h,g)=A f(A^\top Z,\tau,h,g)
+r_\tau(\Pi_\perp Z,h,g),
\]
where
\[
r_\tau(z_\perp,h,g)
=\mathbb{E}[R_t\mid\Pi_\perp Z_\tau=z_\perp,H_t=h,\mathcal{G}=g].
\]
For every $\eta>0$, there exists a ReLU network $f_\theta$ such that
$F_{A,\theta}(Z,\tau,h,g)=A f_\theta(A^\top Z,\tau,h,g)$ obeys
\[
\sup_{\substack{\tau\in[\tau_0,1-\tau_0]\\h\in\mathcal{H},\,g\in\mathcal{G}}}
\|F_{A,\theta}(\cdot,\tau,h,g)-F^*(\cdot,\tau,h,g)\|_{L^2(P_{\tau\mid h,g})}
\le\sqrt{(dP+1)\eta^2+\delta^2}.
\]
\end{corollary}

\begin{proof}
Projecting the noisy state yields
\begin{align*}
A^\top Z_\tau&=\tau S_t+(1-\tau)A^\top\epsilon,\\
\Pi_\perp Z_\tau&=\tau R_t+(1-\tau)\Pi_\perp\epsilon.
\end{align*}
The two Gaussian-noise components are independent. Together with
$S_t\perp R_t\mid(H_t,\mathcal{G})$ and noise independence, this implies
that the pair $(S_t,A^\top Z_\tau)$ and the pair
$(R_t,\Pi_\perp Z_\tau)$ are conditionally independent. Conditioning
$AS_t+R_t$ on the two observation channels therefore gives the stated
decomposition. Conditional Jensen's inequality and the tower property give
\[
\mathbb{E}\!\left[\|r_\tau(\Pi_\perp Z_\tau,h,g)\|_F^2\mid H_t=h,\mathcal{G}=g\right]
\le\mathbb{E}\!\left[\|R_t\|_F^2\mid H_t=h,\mathcal{G}=g\right]
\le\delta^2.
\]
Steps 2--4 of the proof of Theorem~\ref{thm:lowddenoiser} apply to the
projected observation channel and provide one network $f_\theta$ with
latent mean-square approximation error at most $(dP+1)\eta^2$, uniformly
over $(\tau,h,g)$. Since $A^\top R_t=0$ implies $A^\top r_\tau=0$,
the errors are pointwise orthogonal in the Frobenius inner product:
\[
\|F_{A,\theta}-F^*\|_F^2
=\|f_\theta(A^\top Z,\tau,h,g)-f(A^\top Z,\tau,h,g)\|_F^2
+\|r_\tau(\Pi_\perp Z,h,g)\|_F^2.
\]
Taking conditional expectations, the supremum, and square roots proves
the bound. At $\delta=0$, it reduces to the exact-subspace result.
\end{proof}

\section{Low-Dimensional Concentration and Residual Sensitivity on HCP}
\label{app:w1sensitivity}

To examine the low-dimensional approximation empirically, we performed
PCA on the HCP BOLD signals used in our experiments. In the 360-dimensional
ROI space, 21, 65, and 117 principal components explain approximately
80\%, 90\%, and 95\% of the variance, respectively. These results support approximate low-dimensional concentration; the
exact-subspace setting of Theorem~\ref{thm:lowddenoiser} serves as a
stylized model rather than a literal description of real BOLD signals.

For the sensitivity analysis, we decomposed the signals as $X=X_k+R$,
where $X_k$ is the rank-$k$ PCA approximation with $k=65$. We then trained
the three prediction parameterizations on $X_\alpha=X_k+\alpha R$ at each
tested value of $\alpha$. Thus, $\alpha=1$ recovers the original signal,
while larger $\alpha$ increases the off-subspace residual relative to the
fixed low-rank component. The FC correlations are reported in
Table~\ref{tab:w1residual}.

\begin{table}[htbp]
\centering
\caption{FC correlation under residual rescaling on HCP (higher is better).}
\label{tab:w1residual}
\begin{tabular}{cccc}
\toprule
$\alpha$ & FAST-Brain & $\epsilon$-prediction & $v$-prediction\\
\midrule
0.50 & $\mathbf{0.944}\pm0.007$ & $0.785\pm0.005$ & $0.867\pm0.010$\\
1.00 & $\mathbf{0.938}\pm0.005$ & $0.799\pm0.009$ & $0.856\pm0.014$\\
1.50 & $\mathbf{0.922}\pm0.007$ & $0.710\pm0.004$ & $0.808\pm0.008$\\
2.00 & $\mathbf{0.899}\pm0.005$ & $0.810\pm0.005$ & $0.818\pm0.009$\\
\bottomrule
\end{tabular}
\end{table}

FAST-Brain achieves the highest FC correlation at every tested residual
scale, although its correlation decreases as the residual is amplified.
The experiment supports robustness within this controlled construction.
It does not establish the conditional-independence assumption used in
Corollary~\ref{cor:w1approxsubspace}.

\section{Prediction-Interval Coverage}
\label{app:w4coverage}

For each held-out history window, we sample 100 future trajectories from
the trained FAST-Brain model using the sampling procedure of the main
experiments. For each ROI and future time step, we construct central
50\%, 80\%, and 95\% prediction intervals using the empirical
[25th, 75th], [10th, 90th], and [2.5th, 97.5th] percentiles, respectively.
Empirical coverage is the proportion of observed future BOLD values
falling within the corresponding intervals, averaged over test contexts,
ROIs, and prediction horizons.

\begin{table}[H]
\centering
\caption{Empirical coverage of FAST-Brain prediction intervals.
Column headings indicate nominal coverage levels; entries are empirical
coverage percentages averaged across test contexts, ROIs, and horizons.}
\label{tab:w4coverage}
\setlength{\tabcolsep}{13pt}
\renewcommand{\arraystretch}{1.15}
\begin{tabular}{lccc}
\toprule
Dataset & 50\% interval & 80\% interval & 95\% interval\\
\midrule
RNN   & 48.6\% & 77.8\% & 92.8\%\\
SDDEs & 52.4\% & 80.4\% & 93.0\%\\
HCP   & 46.9\% & 78.8\% & 92.4\%\\
\bottomrule
\end{tabular}
\end{table}

As shown in Table~\ref{tab:w4coverage}, empirical coverage is within
3.1 percentage points of the nominal level in all nine settings, with
modest undercoverage at the 95\% level on all three datasets. This
analysis complements QICE and CRPS by directly assessing average
pointwise coverage. Such intervals summarize uncertainty in future BOLD
values for downstream scientific interpretation.

\section{Robustness to Spatial Priors}
\label{app:q4priors}

On HCP, we perturb graphs as
\[
G_{\eta}^{(m)}=(1-\eta)G^{(m)}+\eta G_{\mathrm{rand}}^{(m)},
\]
using symmetric random graphs $G_{\mathrm{rand}}^{(m)}$ with the same
normalization.

\begin{table}[H]
\centering
\caption{FC recovery on HCP with corrupted or reduced spatial priors.}
\label{tab:q4priors}
\small
\setlength{\tabcolsep}{8pt}
\renewcommand{\arraystretch}{1.15}
\begin{tabular}{lcc}
\toprule
Configuration & FC correlation$\uparrow$ & FC MAE$\downarrow$\\
\midrule
Full priors & $0.938\pm0.005$ & $0.073\pm0.008$\\
SC only & $0.930\pm0.009$ & $0.080\pm0.008$\\
Noise 50\% ($\eta=0.5$) & $0.917\pm0.009$ & $0.119\pm0.011$\\
Random graphs ($\eta=1$) & $0.885\pm0.009$ & $0.122\pm0.011$\\
w/o GCN & $0.867\pm0.015$ & $0.133\pm0.017$\\
\bottomrule
\end{tabular}
\end{table}

Accuracy declines with corruption, but random graphs still outperform
removing the GCN (Table~\ref{tab:q4priors}).

\section{Computational Cost}
\label{app:q5cost}

\begin{table}[H]
\centering
\caption{Computational costs on HCP. Forecast latency is for 24 future
time points ($P=24$); M denotes million parameters.}
\label{tab:q5cost}
\footnotesize
\setlength{\tabcolsep}{4pt}
\renewcommand{\arraystretch}{1.2}
\begin{tabular}{lccccc}
\toprule
Method & Parameters & \shortstack{Isolated\\train time}
& \shortstack{Peak training\\memory} & \shortstack{Forecast latency\\($P=24$)}
& \shortstack{Inference\\memory}\\
\midrule
FAST-Brain & 2.4M & 7.9 min & 1.14 GB & 24.6 ms & 0.07 GB\\
FlowTS & 8.9M & 14.62 min & 0.71 GB & 414.4 ms & 0.10 GB\\
DSFM & 130.7M & 1.48 h & 4.94 GB & 1.03 s & 1.28 GB\\
\bottomrule
\end{tabular}
\end{table}

Dense attention uses $PS$ cross-attention and $P^2$ self-attention
token pairs~\citep{vaswani2017attention}, while dense graph propagation
is quadratic in ROI count with other dimensions fixed.

\section{Hyperparameter Sensitivity}
\label{app:hyperparameters}

Results in Table~\ref{tab:hyperparameters} follow the listed value order;
the middle value in each row is the default.

\begin{table}[H]
\centering
\caption{Hyperparameter sensitivity of FC correlation (higher is better).}
\label{tab:hyperparameters}
\footnotesize
\setlength{\tabcolsep}{4pt}
\renewcommand{\arraystretch}{1.25}
\begin{tabular}{@{}lcccc@{}}
\toprule
Hyperparameter & Values & RNN & SDDEs & HCP\\
\midrule
Forecast horizon $P$ & 12 / 24 / 48 & 0.936 / 0.938 / 0.928 & 0.812 / 0.975 / 0.966 & 0.930 / 0.938 / 0.938\\
Graph order $K$ & 1 / 2 / 3 & 0.928 / 0.938 / 0.926 & 0.962 / 0.975 / 0.882 & 0.929 / 0.938 / 0.929\\
ODE steps & 10 / 20 / 50 & 0.935 / 0.938 / 0.926 & 0.978 / 0.975 / 0.969 & 0.939 / 0.938 / 0.936\\
Clipping constant $\tau_0$ & 0.025 / 0.05 / 0.1 & 0.926 / 0.938 / 0.928 & 0.964 / 0.975 / 0.890 & 0.933 / 0.938 / 0.934\\
\bottomrule
\end{tabular}
\end{table}

RNN and HCP correlations remain close across the tested settings;
SDDEs show larger decreases at $P=12$, graph order $K=3$, and $\tau_0=0.1$.

\section{Comparison with CCEP Responses}
\label{app:ccep}

Following \citet{Luo2025}, we compare HCP-derived SPI with
cortico-cortical evoked potential (CCEP) responses from the public
F-TRACT v2112 atlas using the Glasser 360 parcellation. Across 26,771
covered ROI pairs, FAST-Brain achieves the highest Spearman correlation
among the compared methods ($\rho=0.264$; 95\% CI $[0.250,0.278]$; $p<10^{-4}$; Table~\ref{tab:ccep}), supporting modest
cross-modal agreement with stimulation-response patterns.

\begin{table}[H]
\centering
\caption{Agreement between HCP-derived SPI and CCEP responses.}
\label{tab:ccep}
\small
\setlength{\tabcolsep}{14pt}
\renewcommand{\arraystretch}{1.2}
\begin{tabular}{lc}
\toprule
Method & Spearman $\rho\uparrow$\\
\midrule
NPI & 0.186\\
DSFM & 0.158\\
FAST-Brain w/o GCN & 0.219\\
FAST-Brain & 0.264\\
\bottomrule
\end{tabular}
\end{table}

\newpage
\section*{NeurIPS Paper Checklist}

\begin{enumerate}

\item {\bf Claims}
    \item[] Question: Do the main claims made in the abstract and introduction accurately reflect the paper's contributions and scope?
    \item[] Answer: \answerYes{} 
    \item[] Justification: We made the main claim in the abstract and introduction accurately reflect the paper's contributions and scope. 

\item {\bf Limitations}
    \item[] Question: Does the paper discuss the limitations of the work performed by the authors?
    \item[] Answer: \answerYes{} 
    \item[] Justification: We discussed this in conclusion section.

\item {\bf Theory assumptions and proofs}
    \item[] Question: For each theoretical result, does the paper provide the full set of assumptions and a complete (and correct) proof?
    \item[] Answer: \answerYes{} 
    \item[] Justification: We provide the full set of assumptions and a complete proof.

    \item {\bf Experimental result reproducibility}
    \item[] Question: Does the paper fully disclose all the information needed to reproduce the main experimental results of the paper to the extent that it affects the main claims and/or conclusions of the paper (regardless of whether the code and data are provided or not)?
    \item[] Answer: \answerYes{} 
    \item[] Justification: The information needed to reproduce the main experimental results are shown in the experiment section and supplement materail.

\item {\bf Open access to data and code}
    \item[] Question: Does the paper provide open access to the data and code, with sufficient instructions to faithfully reproduce the main experimental results, as described in supplemental material?
    \item[] Answer: \answerYes{} 
    \item[] Justification: We provide the data and code.

\item {\bf Experimental setting/details}
    \item[] Question: Does the paper specify all the training and test details (e.g., data splits, hyperparameters, how they were chosen, type of optimizer) necessary to understand the results?
    \item[] Answer: \answerYes{} 
    \item[] Justification: We include the training and test details.

\item {\bf Experiment statistical significance}
    \item[] Question: Does the paper report error bars suitably and correctly defined or other appropriate information about the statistical significance of the experiments?
    \item[] Answer: \answerYes{}
    \item[] Justification: We report error bars suitably and correctly defined.

\item {\bf Experiments compute resources}
    \item[] Question: For each experiment, does the paper provide sufficient information on the computer resources (type of compute workers, memory, time of execution) needed to reproduce the experiments?
    \item[] Answer: \answerYes{} 
    \item[] Justification: We provide computer resources in Appendix.
    
\item {\bf Code of ethics}
    \item[] Question: Does the research conducted in the paper conform, in every respect, with the NeurIPS Code of Ethics \url{https://neurips.cc/public/EthicsGuidelines}?
    \item[] Answer: \answerYes{} 
    \item[] Justification: We strictly follow the NeurIPS Code of Ethics.

\item {\bf Broader impacts}
    \item[] Question: Does the paper discuss both potential positive societal impacts and negative societal impacts of the work performed?
    \item[] Answer: \answerYes{} 
    \item[] Justification: We discuss potential positive impacts in Introduction part and there are no negative social impacts.
    
\item {\bf Safeguards}
    \item[] Question: Does the paper describe safeguards that have been put in place for responsible release of data or models that have a high risk for misuse (e.g., pre-trained language models, image generators, or scraped datasets)?
    \item[] Answer: \answerNA{} 
    \item[] Justification: Our paper has no such risks. 

\item {\bf Licenses for existing assets}
    \item[] Question: Are the creators or original owners of assets (e.g., code, data, models), used in the paper, properly credited and are the license and terms of use explicitly mentioned and properly respected?
    \item[] Answer: \answerYes{} 
    \item[] Justification: We cited the original paper that produced the code package or dataset. 

\item {\bf New assets}
    \item[] Question: Are new assets introduced in the paper well documented and is the documentation provided alongside the assets?
    \item[] Answer: \answerNA{} 
    \item[] Justification: We do not release new assets.

\item {\bf Crowdsourcing and research with human subjects}
    \item[] Question: For crowdsourcing experiments and research with human subjects, does the paper include the full text of instructions given to participants and screenshots, if applicable, as well as details about compensation (if any)? 
    \item[] Answer: \answerNA{} 
    \item[] Justification: Our paper does not involve crowdsourcing experiments nor research with new subjects.

\item {\bf Institutional review board (IRB) approvals or equivalent for research with human subjects}
    \item[] Question: Does the paper describe potential risks incurred by study participants, whether such risks were disclosed to the subjects, and whether Institutional Review Board (IRB) approvals (or an equivalent approval/review based on the requirements of your country or institution) were obtained?
    \item[] Answer:\answerNA{} 
    \item[] Justification: Our paper does not involve crowdsourcing experiments nor research with new subjects.

\item {\bf Declaration of LLM usage}
    \item[] Question: Does the paper describe the usage of LLMs if it is an important, original, or non-standard component of the core methods in this research? Note that if the LLM is used only for writing, editing, or formatting purposes and does \emph{not} impact the core methodology, scientific rigor, or originality of the research, declaration is not required.
    \item[] Answer: \answerYes{} 
    \item[] Justification: We used LLM for editing, grammar, spelling and word choice.

\end{enumerate}

\end{document}